\documentclass{article}

\usepackage{arxiv}

\usepackage[utf8]{inputenc} 
\usepackage[T1]{fontenc}    
\usepackage{hyperref}       
\usepackage{url}            
\usepackage{booktabs}       
\usepackage{amsfonts}       
\usepackage{nicefrac}       
\usepackage{microtype}      
\usepackage{lipsum}		
\usepackage{graphicx}
\usepackage{natbib}
\usepackage{doi}
\usepackage{amsmath}
\usepackage{amssymb}
\usepackage{xcolor}
\usepackage{mathtools}
\usepackage{amsthm}
\usepackage{enumitem}
\usepackage{svg}
\usepackage{algorithm2e}
\usepackage{multicol, multirow}
\usepackage{arydshln}
\usepackage{booktabs}
\usepackage{cancel}
\usepackage{float}
\usepackage{tikz}

\newcommand{\E}[2]{\mathbb{E}_{#1}\left[#2\right]}
\newcommand{\abs}[1]{
\left\vert
#1
\right\vert
}
\newcommand{\norm}[1]{
\left\Vert
#1
\right\Vert
}

\newtheorem{prop}{Proposition}
\newtheorem{remark}{Remark}

\newtheoremstyle{break}
  {\topsep}{\topsep}%
  {\itshape}{}%
  {\bfseries}{}%
  {\newline}{}%
\theoremstyle{definition}
\newtheorem{definition}{Definition}

\title{
Neural Information field filter 
}
\newcommand{\qref}[1]{Eq.~(\ref{#1})}

\author{{\includegraphics[scale=0.06]{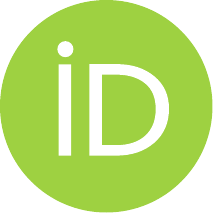}\hspace{1mm}Kairui~Hao\thanks{Corresponding author.}} \\
	School of Mechanical Engineering\\
	Purdue University\\
	West Lafayette, IN \\
	\texttt{hao55@purdue.edu} \\
	\And
	{\includegraphics[scale=0.06]{orcid.pdf}\hspace{1mm}Ilias~Bilionis} \\ 
	School of Mechanical Engineering\\
	Purdue University\\
	West Lafayette, IN \\
	\texttt{ibilion@purdue.edu} \\
}

\renewcommand{\shorttitle}{Neural Information Field Filter}

\hypersetup{
pdftitle={},
pdfsubject={q-bio.NC, q-bio.QM},
pdfauthor={David S.~Hippocampus, Elias D.~Striatum},
pdfkeywords={First keyword, Second keyword, More},
}

\begin{document}
\maketitle

\begin{abstract}
    We introduce neural information field filter, a hierarchical Bayesian state and parameter estimation method for high-dimensional nonlinear dynamical systems given large measurement datasets.
    Traditional methods, such as Kalman and particle filters, are often computationally expensive for such applications. 
    Information field theory offers a Bayesian framework that efficiently reconstructs dynamical model state paths and calibrates model parameters from noisy data.
    To apply the method, we begin by parameterizing the time evolution state path using the span of a finite linear basis.
    This linear basis function should be reparameterized by the initial state to enforce an initial condition.
    Next, we define a physics-informed conditional prior for the state path parameterization given the initial state and model parameters.
    With a likelihood function connecting the unknown quantities to an experiment dataset,  we update the posterior distribution of the state path parameterization and model parameters.
    Designing an expressive yet simple linear basis is essential for inference accuracy, but challenging.
    Moreover, reparameterizing the state path using the initial state is straightforward for a linear basis, but nontrivial for more complex and expressive function parameterizations, such as neural networks.
    The objective of this paper is to simplify and enrich the class of state path parameterizations using neural networks for the information field theory approach to Bayesian state and parameter estimation in dynamical systems.
    To this end, we propose a generalized physics-informed conditional prior using an auxiliary initial state. 
    We show the existing reparameterization is a special case.
    We parameterize the state path using a residual neural network that consists of a linear basis function and a Fourier encoding fully connected neural network residual function.
    The residual function aims to correct the error of the linear basis function.
    To sample from the intractable posterior distribution, we develop an optimization algorithm, nested stochastic variational inference, and a sampling algorithm, nested preconditioned stochastic gradient Langevin dynamics.
    A series of numerical and experimental examples verify and validate the proposed method.
\end{abstract}

\keywords{
Information field theory 
\and
Physics-informed neural networks
\and
Dynamical system
\and
Bayesian state and parameter estimation
\and
Uncertainty quantification
}

\section{Introduction}

Bayesian probabilistic state reconstruction and model parameter estimation for dynamical systems governed by ordinary differential equations are ubiquitous in mathematical and engineering problems.
Such problems include
nonlinear energy sink device~\citep{lund2020identification},
structural dynamics~\citep{chatterjee2023sparse,nayek2023identification},
structural damage identification~\citep{li2023robust},
polymer
composites~\citep{thomas2022bayesian},
magnet synchronous machines~\citep{beltran2020uncertainty},
wind turbines~\citep{song2018wind},
building structures~\citep{kosikova2023bayesian},
fault detection and diagnosis~\citep{murali2024habsim},
building energy~\citep{yi2021model},
particle tracking velocimetry~\citep{hao2023unbalanced, hao2024unbalanced, 10.1088/1361-6501/ad6624}
,
and
heating, ventilation, and air conditioning systems~\citep{hao2020comparing,hao2022comparing}. 
The goal is to use noisy measurement data to estimate unknown time evolution states and model parameters,
incorporating uncertainty quantification and propagation capabilities.
Key steps include specifying a prior distribution for the states and parameters, evaluating the likelihood given measurement data by running forward models, and computing the posterior distribution of the states and parameters.

Standard approaches, such as Kalman filters~\citep{kalman1960new, wan2000unscented, evensen2003ensemble, anderson2012optimal}, particle filters~\citep{liu1998sequential, doucet2001sequential, doucet2000sequential}, and variational filters~\citep{lund2021variational}, can estimate the posterior distribution of dynamical model states from noisy measurements when model parameters are given.
To jointly estimate model states and parameters, we apply methods such as the dual Kalman filter~\citep{wan1996dual, wan2001dual} and nested particle fitlers~\citep{chopin2013smc2, crisan2018nested}.
\citet{kantas2015particle} provides a comprehensive review of particle filters for joint parameter and state estimation in state-space models.
The authors discussed the complexity of using particle filters to estimate model parameters.
The naive approach, i.e., the augmented state method, treats model parameters as additional state variables and applies standard state filtering techniques, such as sequential Monte Carlo. 
However, this approach has been shown to be problematic~\citep{kitagawa1998self}, as it inadequately explores the parameter space.
A more theoretically sound approach considers the hierarchical structure from model parameters to hidden states, and from the hidden states to observations.
Numerical methods for this include maximum likelihood~\citep{hurzeler2001approximating, malik2011particle, dejong2013efficient, klaas2006fast} and Bayesian approaches~\citep{fearnhead2010sequential, andrieu2010particle, chopin2002sequential, fulop2013efficient, liu2001combined, flury2011bayesian}.
When likelihood function is intractable, one can apply likelihood-free methods, such as approximate Bayesian computation~\citep{rubin1984bayesianly, stigler2010darwin, sisson2018handbook}, and Bayesian synthetic likelihood~\citep{price2018bayesian, an2020robust}, to approximate the likelihood using simulated data and a statistical distance, e.g., Wasserstein distance~\cite{kantorovich1960mathematical} and maximum mean discrepancy~\cite{smola2007hilbert}.
Likelihood-free methods have also been successfully integrated into particle filters~\citep{frigola2013bayesian, sisson2007sequential}.
Despite the great success of these standard methods in the past,  they do not leverage modern deep learning software, such as PyTorch~\cite{paszke2019pytorch} and JAX~\citep{jax2018github}, which utilize automatic differentiation~\citep{paszke2017automatic}.
Applying standard methods requires discretizing ordinary differential equations and repeatedly running ODE solvers to evaluate the predictive distribution and likelihood function.
Consequently, the standard approaches do not scale well with the dimensions of model states and the size of measurement data.

Information field theory (IFT) \citep{ensslin2009information, ensslin2013information, ensslin2019information,alberts2023physics, hao2024physics}
is a Bayesian approach to reconstruct infinite-dimensional physical fields, such as pressure and velocity fields.
IFT applies statistical field theory to encode prior knowledge about the field, such as space-time homogeneity, temporal causality, and locality \citep{frank2021field, westerkamp2021dynamical}.
The likelihood function is then constructed using a measurement response function that maps the physical fields to the measurement data.
In general, computing the posterior distribution of the physical fields is intractable, except in special cases, such as a Gaussian random field with a linear measurement response function \citep{lancaster2014quantum}.
Therefore, numerical approaches, such as metric Gaussian variational inference \citep{knollmuller2019metric}, are required to approximate the posterior distribution.

\citet{hao2024information} introduced an information field theory approach to dynamical system state reconstruction and model parameter estimation, which leverages JAX software to accelerate numerical computation.
In this method, the authors parameterize the time evolution state path using a finite number of linear bases.
They then define a prior distribution for the state path parameterization and dynamical model parameters.
This prior has a physics-informed conditional prior for the state path parameterization, given the initial state and model parameters.
IFT constructs this prior using the path integral technique \cite{feynman2010quantum, zinn2021quantum}, which is similar to energy-based models \citep{grenander1994representations}.
This conditional prior introduces a hierarchical structure from the initial state and model parameters to the state path function, a key feature in state space model inference.
Second, a likelihood function relates the state path function and model parameters to a measurement dataset.
Finally, by applying the Bayes' rule, we derive the posterior distribution of the parameterized state path function and model parameters.
In general, the posterior is analytically intractable, requiring numerical methods for approximation.
\citet{alberts2023physics} and \citet{hao2024information} developed 
sampling and optimization approaches using stochastic gradient Langevin dynamics \citep{welling2011bayesian} and
stochastic variational inference \citep{hoffman2013stochastic}, respectively.
The proposed method is computationally efficient and scalable for the following reasons.
First, similar to physics-informed neural networks~\citep{raissi2019physics}, IFT randomly samples a collection of time points to evaluate the physics-informed conditional prior, which is efficiently implemented using \textit{vmap} and \textit{jit} in JAX.
The number of sampling time points is typically much less than the discretized time grid required for recursive ODE solvers.
Second, IFT subsamples a minibatch dataset from a large measurement dataset, enhancing its scalability.  
Despite the promising applications of IFT, we notice that \citet{hao2024information} reparameterizes the state path using the initial state.
This reparameterization trick is straightforward for simple function parameterizations, such as the span of a finite linear basis.
However, it becomes non-trivial for more complex and expressive parameterizations, such as neural networks.
Consequently, the need for reparameterization complicates and limits the choice of state path function representation.

The objective of this paper is to simplify and enrich the class of state path parameterizations using neural networks for the information field theory approach to Bayesian state and parameter estimation in dynamical systems.
We call the method neural information field filter (NIFF).
We introduce a generalized physics-informed conditional prior that does not require reparameterizing the state path with its initial state.
We achieve this using an auxiliary initial state, which is not necessarily the same as the initial state of the state path function.
We use a kernel Hamiltonian to measure the similarity between the auxiliary initial state and the initial state of the parameterized state path function.
We show that the reparameterization trick using the initial state is a special case of the generalized physics-informed conditional prior when the kernel is the Dirac function.
Specifically,  we define a relaxed physics-informed conditional prior by choosing a Gaussian kernel.
We parameterize the state path function using residual neural networks~\citep{he2016deep} that consist of a linear basis function and a residual function.
The linear basis function follows the specification in \citep{hao2024information}, and the residual function is a fully connected neural network with a Fourier encoding first-layer~\citep{tancik2020fourier, hennigh2021nvidia}.
The linear basis function inherits all the advantages of \citep{hao2024information}, such as simplicity in its mathematical form.
However, designing a linear basis to achieve an acceptable accuracy is challenging, as we do not know all the properties, such as the regularity, of the unknown state path functions.
The residual function complements the linear basis function by correcting remaining errors and fine-tuning the estimated state path function. 
To numerically approximate the intractable posterior distribution of the state path parameterization and model parameters, we develop an optimization algorithm, nested stochastic variational inference, and a sampling algorithm, nested preconditioned stochastic gradient Langevin dynamics.
Both methods apply Monte Carlo sampling techniques to efficiently update the unknown quantities.
Finally, we verify and validate NIFF through a series of numerical and experimental examples.

The structure of this paper is as follows.
In section~\ref{section:background}, we review the information field theory approach to Bayesian state and parameter estimation in dynamical systems developed in~\citep{hao2024information}.
In section~\ref{section: methodology}, we theoretically develop NIFF.
In section~\ref{sec:numerical algs}, we develop an optimization algorithm, nested stochastic variational inference, and a sampling algorithm, nested preconditioned stochastic gradient Langevin dynamics.
In section~\ref{section: example}, we verify and validate NIFF through a series of numerical and experimental examples.
Finally, section~\ref{sec: conclusions} concludes the paper.

\section{Background on information field theory approach to Bayesian state and parameter estimation in dynamical systems}
\label{section:background}

We review our previous work on the information field theory approach to dynamical system state reconstruction and parameter estimation~\citep{hao2024information}.
In the following mathematical exposition, we adjust our original notation to emphasize the reparameterization step.

We consider the dynamical system governed by the ODE
\begin{align}
    \dot{x}(t) &= f(x(t), t;\theta),\label{ode}
    \\
    Y(t_k) &= R(x(t_k); \theta) + \text{noise},\nonumber
\end{align}
where 
$x(t)\in \mathbb{R}^{d_x}$ is the state vector, 
$f$ is the vector field,
$\theta\in \mathbb{R}^{d_{\theta}}$ are the model parameters, 
$R$ is the measurement response function, and
$Y(t_k)\in \mathbb{R}^{d_y}$ is the random output vector.
The noise is typically independent, identically distributed, zero-mean Gaussian.
The objective is to estimate the model parameters $\theta$ and time evolution state path $x(t)$ from noisy measurement data $y=(y(t_1), \cdots, y(t_{n_d}))$ at $n_d$ sampling time points.

\citet{hao2024information} parameterizes the state path function $x(t)=\hat{x}(t;w)$ using a $K+1$ linear basis 
$\psi(t)
=\left[
        \psi_0(t), 
        \cdots,
        \psi_{K}(t)
\right]^T$ 
with the coefficient matrix 
$W=\left[w_0, \cdots, w_K\right]\in\mathbb{R}^{d_x\times (K+1)}$, where each $w_i$ is a $d_x$-dimensional column vector.
Notice that we vectorize this matrix into $w=\operatorname{vec}(W)$,
since it is more natural to define a probability distribution for a random vector $w$ than a random matrix $W$.
Then, the parameterized path is
$
    \hat{x}(t; w) = \sum_{i} w_i\psi_i(t).
$
We use this parameterized state path to define the physics-informed prior information Hamiltonian:
\begin{align*}
    H(w, \theta) 
    = 
    \int_0^T dt\ \lVert
     \dot{\hat{x}}(t;w)-f(\hat{x}(t;w), t;\theta) \lVert^2.
\end{align*}

Since IFT requires evaluating and sampling from the physics-informed conditional prior, which is conditioned on the initial state,
we reparameterize $\hat{x}(t; w)$ explicitly using the initial state $x_0\in \mathbb{R}^{d_x}$ (refer to section 2.2 in \citep{hao2024information}).
We achieve this by solving the system of equations:
\begin{align*}
    \sum_{i=0}^{K}
    w_i \psi_i(0) = x_0.
\end{align*}
We select $K$ free bases from $\psi(t)$ and the remaining one basis $\psi_i(t)$ to be dependent.
The dependent coefficient vector is
\begin{align}
    w_i = \frac{
    x_0-\sum_{j\neq i}w_j\psi_j(0)
    }{\psi_i(0)}.
    \label{wj}
\end{align}
Using \qref{wj}, we can compute the dependent coefficient vector $w_i$ given an initial state $x_0$ and free coefficients $w_{j\neq i}$ such that the state path function satisfies the initial condition.
We use the notation $w_{-i}$ to denote the remaining free coefficient vector, i.e., 
$w_{-i}=\operatorname{vec}(\left[w_0, \cdots, w_{i-1}, w_{i+1}, \cdots, w_K\right])$, and concisely denote the function in \qref{wj} by $w_i = \mathcal{T}(x_0; w_{-i})$.
For simplicity of notation,  and without ambiguity, we denote the coefficient vector $w=(w_{-i}, w_i)$ using the free coefficient vector $w_{-i}$ and the dependent coefficient vector $w_i$.
This allows us to express the state path function $\hat{x}(t; w)$ by $\hat{x}(t; w_{-i}, w_i)$.
Finally, we define the reparameterized state path:
\begin{align*}
    \Tilde{x}(t; w_{-i}, x_0)
    &=
    \hat{x}(t; w_{-i}, 
    \mathcal{T}(x_0; w_{-i}))
    \\
    &=
   \mathcal{T}(x_0; w_{-i})
   \psi_i(t)
    +
    \sum_{j\neq i} w_j\psi_j(t).
\end{align*}
It is straightforward to verify that $\Tilde{x}$ automatically satisfies the initial condition, i.e., $\Tilde{x}(0;w_{-i}, x_0) = x_0$, due to \qref{wj}.
The reparameterized information Hamiltonian is
\begin{align*}
    \Tilde{H}(w_{-i}, x_0, \theta)
    &=
    \int_0^T  \lVert
     \dot{\Tilde{x}}(t;w_{-i}, x_0)-f(\Tilde{x}(t; w_{-i}, x_0), t;\theta) \lVert^2
     \ dt
     \\
     &=
     \int_0^T \left\Vert
     \dot{\hat{x}}(t; w_{-i}, \mathcal{T}(x_0; w_{-i}))
     -
     f\left(
     \hat{x}(t; w_{-i}, \mathcal{T}(x_0;w_{-i})), t; \theta
     \right)
     \right\Vert^2 \ dt
     \\
     &=
    H(w_{-i}, \mathcal{T}(x_0; w_{-i}), \theta).
\end{align*}

The objective of IFT is to find the posterior distribution:
\begin{align}
    p(w_{-i}, x_0, \theta\vert y)
    =
    \frac{
    p(y\vert w_{-i}, x_0, \theta)
    \Tilde{p}(w_{-i}\vert x_0, \theta)
    p(x_0, \theta)
    }{
    p(y)
    },
    \label{ift_post}
\end{align}
where the physics-informed conditional prior is: 
\begin{align}
    \Tilde{p}(w_{-i}\vert x_0, \theta)
    =
    \frac{
    e^{-\beta \Tilde{H}(w_{-i}, x_0, \theta)}
    }{
    Z(x_0, \theta)
    }.\label{hao24prior}
\end{align}
In the above definition, 
$\beta$ is a hyperparameter that controls our level of trust in the model.
A greater $\beta$ enforces the physics harder, and the normalization constant
\begin{align*}
    Z(x_0, \theta)
    =
    \int
    e^{-\beta \Tilde{H}(w_{-i}, x_0, \theta)}
    \
    dw_{-i}
\end{align*}
is the partition function.
The essence of IFT is encoding well-known physics through the physics-informed conditional prior.

The reparameterization step is essential because, given the initial state $x_0$ and model parameters $\theta$, the ODE defined in Eq.~(\ref{ode}) has a unique solution under mild conditions~\citep{meiss2007differential}.
This uniqueness makes it feasible to sample from the physics-informed condition prior $\Tilde{p}(w_{-i}\vert x_0, \theta)$.
However, reparameterizing the state path function requires solving the system of equations in Eq.~(\ref{wj}).
This process is straightforward when $\hat{x}(t; w)$ is defined by a finite linear basis, but becomes complex for more expressive function parameterizations, such as neural networks.
In the following, we extend IFT to eliminate the need for reparameterization, enabling us to use neural networks to enrich the state path parameterization.

\section{Theoretical development}
\label{section: methodology}

\subsection{Generalized physics-informed conditional prior}
We parameterize the state path using $\hat{x}(t;w)$ without reparameterization.
First, we make the following definition.
\begin{definition}[Generalized physics-informed conditional prior]
    Let $x_0$ be the auxiliary initial state variable, which is not necessarily equal to $\hat{x}(0;w)$.
    The generalized physics-informed conditional prior is
    \begin{align*}
        p(w\vert x_0, \theta)
        =
        \frac{
        e^{
        -\beta H(w,  \theta)
        -H(\hat{x}(0;w), x_0)
        }
        }{
        Z(x_0, \theta)
        },
\end{align*}
where the kernel Hamiltonian is
\begin{align*}
    H(\hat{x}(0;w), x_0)
    =
    -\log K(\hat{x}(0;w), x_0),
\end{align*}
and the normalization constant is
\begin{align*}
    Z(x_0, \theta)
    =
    \int 
    e^{
        -\beta H(w,  \theta)
        -H(\hat{x}(0;w), x_0)
        }\
         dw.
\end{align*}
\end{definition}
If the kernel Hamiltonian is not properly defined, e.g., when $K(\cdot, \cdot)$ is the Dirac function, we simply write 
$p(w\vert x_0, \theta)=K(\hat{x}(0;w), x_0)\frac{e^{-\beta H(w, \theta)}}{Z(x_0, \theta)}$.

To obtain the conditional prior $p(w\vert \theta)$, we marginalize out $x_0$ using a prior distribution for the auxiliary initial state:
\begin{align*}
    p(w\vert \theta)
    =
    \int 
    p(w\vert x_0, \theta)
    p(x_0)\ dx_0.
\end{align*}
This step essentially performs a convolution with the kernel
$
K(
\hat{x}(0;w),x_0
)
$.

The prior defined in Eq.~(\ref{hao24prior}) is a special case of the generalized physics-informed conditional prior when the Kernel $K(\hat{x}(0;w), x_0)$ is the Dirac function.
\begin{prop}
    \label{prop1}
    Choose any $i$ such that $\psi_i(0)\ne 0$.
    Denote the marginal distribution of the generalized physics-informed conditional prior $p(w_{-i}\vert x_0, \theta)=\int p(w\vert x_0, \theta)\ dw_i$, and the joint distribution of the reparameterized prior  $\Tilde{p}(w\vert\theta)$.
    Define the function $w_i=\mathcal{T}(x_0; w_{-i})$ such that 
    $
    \hat{x}(0; w_{-i}, \mathcal{T}(x_0; w_{-i}))=x_0
    $, and assume $\mathcal{T}$ is bijective between $x_0$ and $w_i$.
    If the kernel is
    $
    K\left(\hat{x}(0;w),x_0\right)
    =
    \delta(
    w_i-\mathcal{T}(x_0;w_{-i})
    )
    $, then $p(w_{-i}\vert x_0, \theta)=\Tilde{p}(w_{-i}\vert x_0, \theta)$
    and 
    $
    p(w\vert \theta) = \Tilde{p}(w\vert \theta)
    $.
    Thus, we recover the reparameterization approach described in \citep{hao2024information}.
\end{prop}

Proof is in Appendix~\ref{appendix:proof1}.

\subsection{Relaxed physics-informed conditional prior}

We choose the special Kernel
\begin{align*}
    K(\hat{x}(0;w), x_0)= e^{-\beta_{2}\norm{\hat{x}(0;w)- x_0}^2},
\end{align*}
and use 
$
H(\hat{x}(0;w), x_0)
=
\norm{\hat{x}(0;w)- x_0}^2
$
to denote the kernel  Hamiltonian.
We then define the relaxed physics-informed conditional prior:
\begin{align}
\label{relaxed_prior}
    p(w\vert x_0, \theta)
    =
    \dfrac{
    e^{
    -\beta_1 H(w, \theta)
    -\beta_2 H(\hat{x}(0;w), x_0)
    }
    }{
    Z(x_0, \theta)
    }.
\end{align}

The normalization constant is
\begin{align}
\label{partition}
    Z(x_0, \theta)
    =
    \int 
    e^{
        -\beta_1 H(w, \theta)
        -\beta_2 H(\hat{x}(0;w), x_0)
    }\ dw.
\end{align}

\begin{remark}
    The relaxed physics-informed conditional prior has a probabilistic interpretation of the proximal operator~\citep{parikh2014proximal}.
    It assigns higher probability when $H(w, \theta)$ is smaller and the initial state of the state path, $\hat{x}(0;w)$, is closer to the auxiliary initial state $x_0$.
\end{remark}

By applying Bayes' rule and integrating out the auxiliary initial state $x_0$, the posterior distribution is given by
\begin{align}
    p(w, \theta\vert y)
    &=
    \int p(w, \theta, x_0\vert y)
    \
    dx_0
    \nonumber
    \\
    &=
    \int
    \frac{
    p(y\vert w, \theta)p(w\vert x_0, \theta)p(\theta)p(x_0)
    }{
    p(y)
    }
    \ dx_0.\label{marg_posterior}
\end{align}

We plot three directed acyclic graphs in Fig.~\ref{fig:dag} to visualize the hierarchically structural differences of Bayesian PINNs, IFT with the reparameterized state path approach described in~\citep{hao2024information}, and our proposed method, i.e., NIFF with the relaxed physics-informed conditional prior approach.
We use circles to denote nodes that are randomly generated from their parent nodes, and squares to denote nodes that are deterministically generated from their parent nodes.
The additional shaded node $y_f$ (\citet{yang2021b} denote it by $\mathcal{D}_f$) in Bayesian PINNs is the fictitious observation for enforcing physics.
\begin{remark}
    In our relaxed method, the auxiliary initial state $x_0$ does not have a direct link to the parameterized state path function.
\end{remark}

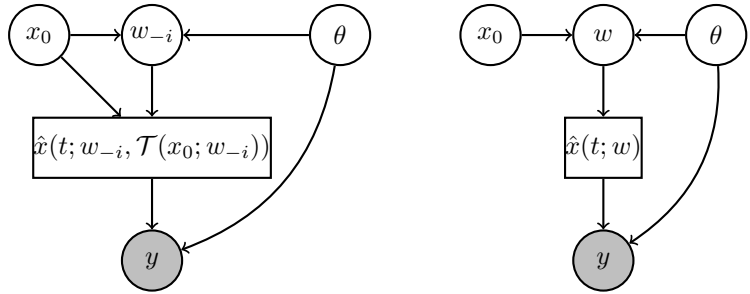
\begin{figure}[hbt]
    \centering
    \begin{tikzpicture}[style={thick}]
        \node [circle,draw=black,fill=white,inner sep=0pt,minimum size=0.8cm] (x0) at (0,1.5) { $x_0$};
        \node [circle,draw=black,fill=white,inner sep=0pt,minimum size=0.8cm] (theta) at (4.5, 1.5) { $\theta$};
        \node [circle,draw=black,fill=white,inner sep=0pt,minimum size=0.8cm] (w) at (1.5, 0) {$w_{-i}$};
        \node [rectangle,draw=black,fill=white,inner sep=0pt,minimum size=0.8cm] (hatx) at (1.5, -1.5) {$\hat{x}(t;w_{-i}, \mathcal{T}(x_0; w_{-i}))$};
        \node [circle,draw=black,fill=lightgray,inner sep=0pt,minimum size=0.8cm] (y) at (1.5, -3) {$y$};
        \path [draw, ->] (x0) edge (w);
        \path [draw, ->] (theta) edge (w);
        \path [draw, ->] (x0) edge (hatx);
        \path [draw, ->] (w) edge (hatx);
        \path [draw, ->] (hatx) edge (y);
        \path [draw, ->] (theta)[bend left]  edge (y);
        \node [circle,draw=black,fill=white,inner sep=0pt,minimum size=0.8cm] (x0new) at (6,1.5) { $x_0$};
        \node [circle,draw=black,fill=white,inner sep=0pt,minimum size=0.8cm] (thetanew) at (9, 1.5) { $\theta$};
        \node [circle,draw=black,fill=white,inner sep=0pt,minimum size=0.8cm] (wnew) at (7.5, 0) {$w$};
        \node [rectangle,draw=black,fill=white,inner sep=0pt,minimum size=0.8cm] (hatxnew) at (7.5, -1.5) {$\hat{x}(t;w)$};
        \node [circle,draw=black,fill=lightgray,inner sep=0pt,minimum size=0.8cm] (ynew) at (7.5, -3) {$y$};
        \path [draw, ->] (x0new) edge (wnew);
        \path [draw, ->] (thetanew) edge (wnew);
        \path [draw, ->] (wnew) edge (hatxnew);
        \path [draw, ->] (hatxnew) edge (ynew);
        \path [draw, ->] (thetanew)[bend left]  edge (ynew);
        \node [circle,draw=black,fill=white,inner sep=0pt,minimum size=0.8cm] (thetab) at (-2,0) { $\theta$};
        \node [circle,draw=black,fill=white,inner sep=0pt,minimum size=0.8cm] (wb) at (-4,0) { $w$};
        \node [rectangle,draw=black,fill=white,inner sep=0pt,minimum size=0.8cm] (hatxbnn) at (-4, -1.5) {$\hat{x}(t;w)$};
        \node [circle,draw=black,fill=lightgray,,inner sep=0pt,minimum size=0.8cm] (yb) at (-4,-3) { $y$};
        \node [circle,draw=black,fill=lightgray,,inner sep=0pt,minimum size=0.8cm] (yrb) at (-2,-3) { $y_f$};
        \path [draw, ->] (thetab) edge (yb);
        \path [draw, ->] (wb) edge (hatxbnn);
         \path [draw, ->] (hatxbnn) edge (yb);
        \path [draw, ->] (thetab) edge (yrb);
        \path [draw, ->] (hatxbnn) edge (yrb);
    \end{tikzpicture}
    \caption{
    Directed acyclic graphs for Bayesian PINNs (left), IFT with the reparameterized state path approach~\citep{hao2024information}) (middle), and NIFF with the relaxed physics-informed conditional prior approach (right).
    }
    \label{fig:dag}
\end{figure}

\begin{remark}
    By introducing the partition function, IFT and NIFF incorporate a hierarchical structure from the initial state or auxiliary initial state,  and model parameters to the parameterized state path function, which is not considered in Bayesian PINNs~\cite{yang2021b}.
\end{remark}
This hierarchical structure is generally preferred in joint state and parameter estimation for state space models~\cite{kantas2015particle}.
For example, in standard batch filtering and parameter calibration algorithms, we formulate the Bayesian inverse problem:
$p(x_{0:n_d}, \theta\vert y_{0:n_d})\propto p(y_{0:n_d}\vert x_{0:n_d}, \theta)p(x_{1:n_d}\vert x_0, \theta)p(x_0, \theta)$.
Comparing this formulation to IFT's formulation in Eq.~\ref{ift_post} and NIFF's formulation in Eq.~\ref{marg_posterior}, we observe that IFT and NIFF replace the discrete state variables $x_{0:d_n}$ with continuous state path functions and use the physics-informed conditional priors to maintain the hierarchical structure.
In contrast, the neural network parameter prior in Bayesian PINNs is usually a simple diagonal multivariate normal distribution.
Please refer to section 2.3 in \citep{hao2024information} for a detailed mathematical comparison.

\subsubsection{Fourier encoding residual neural networks}

We parameterize the state path $x(t) = \hat{x}(t; w)$ using neural networks.
Due to the stochastic relaxation, reparameterization with the initial state $x_0$ is not required.
We adopt a hybrid parameterization that combines a linear basis
$\sum_{i} w_i^b\psi_i(t)$ and 
a neural network, where the neural network is designed to learn the residual state path.
For the neural network component, we first pass the time variable through a Fourier feature encoding layer~\citep{tancik2020fourier, hennigh2021nvidia}, which helps address the spectral bias issue in neural networks \citep{rahaman2019spectral, wang2022and}.
Specifically, we encode the time $t$ to a $2K+1$ high dimensional Fourier space with a $\Bar{T}$ time period as follows:
\begin{align*}
    \left(
    1, 
    \sin{\frac{2\pi t}{\Bar{T}}},
    \cos{\frac{2\pi t}{\Bar{T}}},
    \cdots,
    \sin{\frac{2\pi K t}{\Bar{T}}},
    \cos{\frac{2\pi K t}{\Bar{T}}}
    \right)
    =
    T_{\text{encoder}}(t).
\end{align*}
Let $T_i(z; w^{\text{NN}}_i, b_i^{\text{NN}})=\sigma_i(w_i^{\text{NN}} z + b_i^{\text{NN}})$ denote a neural network hidden layer, where $w_i^{\text{NN}}$ is the weight matrix, $b_i^{\text{NN}}$ is the bias vector, and $\sigma_i$ is the nonlinear activation function.
The hybrid parameterization consists of two parallel predictive paths using the linear basis and neural network functions.
We collectively write $w=(w^{\text{NN}}_{1:n_{\text{hidden}}}, b^{\text{NN}}_{1:n_{\text{hidden}}}, w^{\text{NN}}_{\text{out}}, w^b)$, and
the hybrid parameterization is
\begin{align}
    \operatorname{NN}(t; w)
    =
    T_{\text{out}}\circ
     T_{n_{\text{hidden}}}\circ
    T_{n_{\text{hidden}}-1}\circ
    \cdots
    T_{1}
    \circ
    T_{\text{encoder}}
    (t; w^{\text{NN}}_{1:n_{\text{hidden}}},
    b^{\text{NN}}_{1:n_{\text{hidden}}},w^{\text{NN}}_{\text{out}})
    +
    \sum_{i}
     w_i^{b}\psi_i(t).
    \label{residualfourier}
\end{align}

\section{Numerical algorithms}
\label{sec:numerical algs}
In this section, we develop two numerical algorithms to sample from the posterior distribution $p(w, \theta\vert y)$ defined in \qref{marg_posterior}.
The first one is an optimization-based method called nested stochastic variational inference (NSVI).
The second one is a Markov chain Monte Carlo sampling-based method called nested preconditioned stochastic gradient Langevin dynamics (NPSGLD).
NSVI is computationally efficient and well-suited for high-dimensional problems, such as those with more than ten thousand unknown variables, though it is generally less accurate. 
NPSGLD, on the other hand, theoretically converges asymptotically but has a prohibitively slow mixing time for high-dimensional problems.
A practical guideline for choosing between these methods is as follows: for low-dimensional problems (e.g., fewer than one thousand unknown variables), both NSVI and NPSGLD are appropriate. 
NPSGLD provides better estimation results when NSVI uses a simple approximating probability distribution, such as a diagonal multivariate Gaussian. For high-dimensional problems, we recommend NSVI due to its faster convergence, despite some approximation error.

The relaxed physics-informed conditional prior defined in \qref{relaxed_prior} includes the normalization constant $Z(x_0, \theta)$, which depends on the auxiliary initial state and model parameters.
Since both NSVI and NPSGLD require the gradient of $\log Z(x_0, \theta)$,
we define the unnormalized relaxed physics-informed conditional prior as
$$
    \pi(w\vert x_0, \theta)
    =
    e^{-\beta_1 H_1(w, \theta)-\beta_2 H_2(\hat{x}(0;w), x_0)}
$$
so that we can write
$$p(w\vert x_0, \theta)=\frac{\pi(w\vert x_0, \theta)}{Z(x_0, \theta)}.$$

\subsection{Nested stochastic variational inference}

\subsubsection{Background on variational inference}
In variational inference, the goal is to find an optimal parameterized guide $q_{\phi}(z)$, e.g., a normal distribution, to approximate the posterior distribution $p(z\vert y)=\frac{p(y, z)}{p(y)}$.
Variational inference achieves this by minimizing the Kullback-Leibler (KL) divergence~\citep{kullback1951information} between the guide and the posterior distribution:
\begin{align*}
\min_{\phi}
\quad
    D_{\text{KL}}\left(
     q_{\phi}(z)\Vert 
     p(z\vert y)
    \right)
    =
    \E{q_{\phi}(z)}
    {
    \log 
    \dfrac{
    q_{\phi}(z)
    }{
    p(z\vert y)
    }
    }.
\end{align*}
The KL divergence 
$D_{\text{KL}}\left(
     q_{\phi}(z)\Vert 
     p(z\vert y)
    \right)
$
is nonnegative.
When it equals to zero, $q_{\phi}(z)=p(z\vert y)$ almost everywhere.

However, directly evaluating the KL divergence is computationally infeasible due to the intractability of the evidence $p(y)$.
Instead, we maximize a dual objective function called evidence lower bound (ELBO)~\citep{jordan1999introduction, kingma2013auto}:
\begin{align*}
\max_{\phi}
\quad
    \operatorname{ELBO}\left(\phi\vert y\right)
    =
    \E{q_{\phi}(x)}
    {
        \log
        \dfrac{
        p(y, z)
        }{
        q_{\phi}(z)
        } 
    }.
\end{align*}
Maximizing the ELBO is equivalent to minimizing 
$D_{\text{KL}}\left(
     q_{\phi}(z)\Vert 
     p(z\vert y)
    \right)
$,
and it only requires evaluating $p(y, z)$ rather than $p(z\vert y)$.

\subsubsection{Nested stochastic variational inference to approximate the marginal posterior $p(w, \theta\vert y)$}

Selecting an appropriate guide to approximate the marginal posterior requires balancing expressiveness and computational complexity.
The state path function parameters $w$ are typically high-dimensional, e.g., more than one thousand, while the physical model parameters $\theta$ are low-dimensional, e.g., fewer than ten.
Therefore, to approximate $p(w, \theta\vert y)$, we factorize the guide as $q_{\phi, \psi}(w, \theta)=q_{\phi}(w)q_{\psi}(\theta)$, separating $w$ and $\theta$ for computational efficiency.
To accelerate computation, the high dimensional guide $q_{\phi}(w)$ can be a simple parametric form, e.g., a diagonal multivariate normal distribution. 
The low dimensional guide $q_{\psi}(\theta)$ can be more expressive, e.g., a full-rank normal distribution, to capture correlations between parameters.
Additionally, we specify a guide $q_{\chi}(x_0)$ for the auxiliary initial state.

Typically, measurement data $y$ satisfies the conditionally independent assumption, i.e.,
\begin{align*}
    \log p(y\vert w, \theta)
    =
    \sum_{i=1}^{n_d}
    \log p(y_i\vert  \hat{x}(t_i; w), \theta).
\end{align*}
To handle large datasets, we subsample a minibatch of indices $\mathcal{I}_{m_d}$ of size $m_d$ from the set $\{1, \cdots, n_d\}$ with a probability $
{n_d \choose m_d}^{-1}
$.

The objective of NSVI is to maximize the ELBO:
\begin{equation}
\label{posterior_elbo}
\begin{aligned}
    \max_{\phi, \psi, \chi}
    \quad
    \operatorname{ELBO}(\phi, \psi, \chi\vert y)
    =
    &+
    \E{q_{\phi}(w)q_{\psi}(\theta) q_{\chi}(x_0)
    p(\mathcal{I}_{m_d})
    }{
    \frac{n_d}{m_d}
    \sum_{i\in\mathcal{I}_{m_d}}
    \log p(y_i\vert w, \theta)
    +
        \log
        \left\{
        \dfrac{
        \pi(w\vert x_0, \theta) p(x_0, \theta)
        }{
        q_{\phi}(w)q_{\psi}(\theta) q_{\chi}(x_0)
        }
        \right\}
        }
    \\
    &-\E{q_{q_{\psi}(\theta) q_{\chi}(x_0)}
    }{\log Z(x_0, \theta)}
\end{aligned}.
\end{equation}
Justification for this objective can be found in Appendix~\ref{appendix:proof2}.

Maximizing this ELBO requires taking the gradient of the log partition function.
This is not trivial, as the partition function is defined using a high-dimensional integration \qref{partition}.
We devise an inner loop auxiliary stochastic variational inference, i.e., a nested loop, to sample from the relaxed physics-informed conditional prior.
We leave the detailed numerical implementation steps in Appendix~\ref{appendix:nsvi}.

\subsection{Nested preconditioned stochastic gradient Langevin dynamics}

\subsubsection{Background on preconditioned stochastic gradient Langevin dynamics}

MCMC sampling from a posterior distribution $p(z\vert y)$ using unadjusted overdamped Langevin dynamics \citep{langevin1908theorie} applies the following update step:
\begin{align*}
    \Delta z_k
    =
    \rho_k 
    \left(
        \nabla_{z}
        \log
            p(y\vert z_k)
        +
        \nabla_{z}
        \log
            p(z_k)
    \right)
    +
    \sqrt{2\rho_k} \xi_k,
\end{align*}
where $\rho_k$ is the learning rate, and
$\xi_k$ follows a multivariate diagonal Gaussian
$
\mathcal{N}(0, I)
$.
The step size $\rho_k$ should satisfy the Robbins-Monro conditions \citep{robbins1951stochastic}
\begin{align*}
    \sum_{k}^{\infty} \rho_k = \infty,
    \quad
    \sum_{k}^{\infty} \rho^2_k < \infty
\end{align*}
to converge to a local maximum.

To scale to a large dataset, stochastic gradient Langevin dynamics (SGLD) subsamples the measurement data.
The update step is given by \citep{welling2011bayesian}
\begin{align*}
     \Delta z_k
    =
    \rho_k
    \left(
    \frac{n_d}{m_d}
        \sum_{i=1}^{m_d}
        \nabla_{z}
        \log
            p(y_{ki}\vert z_k)
        +
        \nabla_{z}
        \log
            p(z_k)
    \right)
    +
    \sqrt{2\rho_k} \xi_k.
\end{align*}

One of the issues in stochastic gradient Langevin dynamics is that it applies the same learning rate to all variables. 
This can result in updated step sizes that differ by several orders of magnitude across variables, leading to instability and uneven convergence rates. 
To address this, preconditioned stochastic gradient Langevin dynamics (PSGLD) introduces a preconditioning matrix
$
  M(z_k)
$
that adaptively scales the updates in Langevin dynamics.
An example is stochastic gradient Riemannian Langevin dynamics,  which updates according to~\citep{patterson2013stochastic}
\begin{align*}
    &
    \Delta z_k
    =
    \rho_k
    \left(
        M(z_k)
        \left(
        \frac{n_d}{m_d}
            \sum_{i=1}^{m_d}
            \nabla_{z}
            \log
                p(y_{ki}\vert z_k)
            +
            \nabla_{z}
            \log
                p(z_k)
        \right)
        +\Gamma(z_k)
    \right)
    +
    M(z_k)^{\frac{1}{2}}
    \sqrt{2\rho_k} \xi_k,
\end{align*}
where 
$
\Gamma_i(z_k)
=
\sum_{j}
\partial_j M_{ij}(z_k)
$.
When $M(z_k)$ is full rank, computing $\Gamma(z_k)$ becomes numerically expensive. 
To reduce computation while maintaining effectiveness, we apply a diagonal precondition strategy \citep{li2016preconditioned}, using the RMSprop rule \citep{hinton2012neural}:
\begin{equation*}
    \begin{split}     
    V(z_k)
    &=
    \alpha V(z_{k-1})
    +
    (1-\alpha)
    g(z_k)\odot g(z_k),
    \\
    M(z_k)
    &=
    \operatorname{diag}
    \left(
        \frac{1}{
        \delta +
        \sqrt{
        V(z_k)
        }
        }
    \right).
    \end{split}
\end{equation*}
In the above equations, we define
\begin{align*}
    g(z_k)
    =
    \frac{1}{m_d}
    \sum_{i=1}^{m_d}
    \nabla_{z}
    \log
    p(y_{ki}\vert z_k),
\end{align*}
where $\odot$ denotes element-wise multiplication.
This preconditioning strategy helps to maintain consistent updated step sizes across variables, keeping them within the same order of magnitude.
The two hyperparameters are $\alpha$ which determines the memory size, and $\delta$ which prevents numerical instability when $ V(z_k)$ approaches zero and controls
the extremes of curvature in the preconditioner \citep{li2016preconditioned}.

\subsubsection{NPSGLD to sample from the marginal posterior $p(w, \theta\vert y)$}

To sample from the marginal posterior $p(w, \theta \vert y)$, we first sample $\left\{w_{k}, \theta_k, x_{0,k}\right\}$ from the joint posterior 
$p(w, \theta, x_0\vert y)$, and only retain
$\left\{w_{k}, \theta_k\right\}$ to marginalize out $x_0$.
Using the shorthand notation 
$M_k = M(w_k, x_{0, k}, \theta_k)$ and
$\Gamma_k = \Gamma(w_k, x_{0, k}, \theta_k)$, the update step is given by
\begin{equation}
\label{npsgld_update}
\begin{aligned}
    \left[
        \Delta w_k, 
        \Delta \theta_k,
        \Delta x_{0,k}
    \right]
    =&
    +
    \rho_k 
    \left(
    M_k
    \left(
            \frac{n_d}{m_d}
            \sum_{i=1}^{m_d}
            \nabla_{w, \theta}
            \log
                p(y_{ki}\vert w_k, \theta_k)
            +
            \nabla_{w, \theta, x_0}
            \log
            \frac{
            \pi(w_k\vert x_{0,k}, \theta_k)p(x_{0,k},\theta_k)
            }
            {Z(x_{0, k}, \theta_k)}
    \right)
    +\Gamma_k
    \right)
    \\
    &
     +
    M_k^{\frac{1}{2}}
    \sqrt{2\rho_k} \xi_k.
\end{aligned}.
\end{equation}

As in NSVI, the update rule requires calculating the gradient of the log partition function.
We implement an inner loop auxiliary preconditioned stochastic gradient Langevin dynamics, i.e., a nested loop.
The preconditioning matrix $M_k$ also depends on the log partition function.
Detailed numerical implementation steps are provided in Appendix \ref{appendix: npsgld}.

\subsection{
Inferring the hyperparameter $\beta$
}

The choice of the hyperparameter $\beta$ typically involves either a pragmatic selection or hyperparameter optimization.
In this paper, we adopt the first strategy.
For a detailed discussion on hyperparameter optimization, readers may refer to section 2.5 in \citep{hao2024information}.
Specifically, one can parameterize $\beta$ using a softplus transformation of an unconstrained variational parameter and maximize the likelihood with respect to this variational parameter.

\section{Numerical examples}
\label{section: example}

In this section, we conduct several examples to verify and validate NIFF.
In section~\ref{example 1}, we consider a synthetic example based on a single-degree-of-freedom Duffing oscillator.
This example is the same as one example from \citep{hao2024information}.
We use this example to compare and verify the proposed non-reparameterized state path function approach in this paper with the reparameterized state path function approach published in \citep{hao2024information}.
In section~\ref{section:2dof}, we study the improvement using a residual function to parameterize the state path function.
Specifically, we use a two-degree-of-freedom nonlinear system considered in \citep{kong2022non}.
In section~\ref{section:high}, we demonstrate the performance of NIFF on a high-dimensional twenty-story frame structure model problem.
In section~\ref{section:nes}, we validate NIFF using an experimental nonlinear energy sink problem.
In all examples, we use the same parametric guide, as detailed in section~\ref{example 1} and
we report the 90\% quantile posterior and predictive results.
The computational costs for the examples are provided in Appendix~\ref{appendix: computational time}.

\subsection{Comparison between reparameterized and non-reparameterized state path functions}
\label{example 1}

In this example, we verify the proposed relaxed physics-informed conditional prior approach by comparing it with the reparameterized approach developed in \citep{hao2024information}.
We use the same Duffing oscillator example from section 3.3 in \citep{hao2024information}.

The Duffing oscillator is described by the following equations:
\begin{align}
     \dot{x}_{1}(t) &= x_2 (t), \nonumber
     \\
     \dot{x}_{2}(t) &= -k_1 x_{2}(t) - k_2 x_{1}(t) - k_3 x_{1}(t)^3 + \gamma \cos{(\omega t)}, \nonumber
     \\
     Y(t) &= x_1 (t) + \sigma_y V(t).\nonumber
 \end{align}
This model describes a damped oscillator undergoing a nonlinear restoring force term $k_3 x_1 (t)^3$. 
The oscillator is excited by a cosine signal with known amplitude $\gamma$ 0.37 m and frequency $\omega$ 1.2 rad/s.
To reconstruct the state function and model parameters, we use position measurements perturbed by a scaled white noise process $\sigma_y V(t)$.

The reference true parameter values are 
$k_1 = 0.3$, $k_2 = -1$, and $k_3 = 1$.
To improve the numerical stability, we normalize the state variables, model parameters and measurement data by the constants
$
(\bar{x}_1, \bar{x}_2) = (1.5, 1)
$,
$
(
\bar{k}_1, \bar{k}_2, \bar{k}_3
)
=
(1, 1, 1)
$,
and $\bar{y}=1.5$.

To generate synthetic measurement data, we ran a 50-second forward simulation using the Runge-Kutta method \citep{press2007numerical} with a time step of 0.01s.
The initial states are $(x_1(0), x_2(0))^T=(1, 0)^T$, and the measurement noise standard deviations are set to 5\% of the measurement normalization constant.

We parameterize the state function using the truncated Fourier series
\begin{align}
    \label{truncatedfourier}
    \hat{x}(t; w) = w_0 + \sum_{k=1}^{K} 
    \left(
    w_{2k-1}\sin \frac{2\pi k}{\Bar{T}}
    +
    w_{2k}\cos{\frac{2\pi k}{\Bar{T}}}
    \right)
\end{align}
with $K=40$.
We do not include a residual function, as this example focuses on verifying the proposed non-reparameterized state path function approach.

We compare the posterior distributions approximated by the following four numerical methods.
For all four cases, the prior distributions for the model parameters, initial states, and auxiliary initial states are standard normal distributions. 
The setups are as follows.

First, we use the reparameterized state function approach from~\citep{hao2024information} and solve it using NSVI.
We choose a diagonal normal distribution as the guide for $w$ and $x_0$, and a full-rank normal distribution as the guide for $\theta$.
Readers can refer to section 2.4.4~\citep{hao2024information} for a detailed discussion on guide selection.
We also emphasize that our objective in developing NSVI is to create a more efficient alternative to sampling-based methods. 
While more advanced guides, such as normalizing flow~\citep{rezende2015variational, papamakarios2021normalizing}, can approximate complex distributions,
they significantly increase computational time due to entropy evaluation.
Training such complex guides can be even slower than running MCMC.
These advanced guides are primarily used in applications requiring reusability, such as amortized variational inference~\citep{zhang2018advances} and generative models~\citep{zhang2021diffusion}.
In terms of other settings,
the inverse temperature $\beta=200$, and the remaining optimization settings (e.g., learning rate and sample sizes) follow those in~\citep{hao2024information}.

Second, we use the relaxed physics-informed conditional prior proposed in this paper and solve it with NSVI.
The guide for the Fourier coefficients and the auxiliary initial state is a diagonal normal distribution, while the model parameter guide is a full-rank normal distribution. 
 We set the hyperparameters $\beta_1=200$ and $\beta_2=100,000$.
 The sample sizes used in Algorithm \ref{alg:PSGLD_post} are
 $
 (n_{\epsilon\eta\zeta},n_t,  \tilde{n}_\epsilon, \tilde{n}_t, m_y)
 =
 (1, 10, 1, 10, 10)
 $.
The initial learning rate is 0.001 and we exponentially decay it every 100,000 iterations by a factor of 0.1.
We use the Adam optimizer~\citep{kingma2014adam} to update the guide parameters.
The Adam parameters are set to default, i.e., $b_1=0.9$, $b_2=0.999$.

Third,  we apply the relaxed physics-informed conditional prior and solve it using NSGLD developed in \citep{alberts2023physics}.
We initialize three MCMC chains and sample them in parallel for $2\times 10^6$ steps.
We use the same hyperparameters $\beta_1=200$ and $\beta_2=100,000$.
Sample sizes are the same as in case two.
The learning rates for all parameters are $10^{-6}$ and kept constant throughout the sampling.

Last, we use the relaxed physics-informed conditional prior and solve it using NPSGLD.
Sample sizes are the same as in case two.
Due to the preconditioning matrix, we can use more aggressive learning rates early in sampling.
We set the initial learning rate to $10^{-4}$ and incorporate an exponential decay scheduler that decays the learning rate to $10^{-5}$ after $10^6$ iterations.
We then hold the rate constant.
The preconditioning matrix hyperparameter is set to $\delta=0.1$. The initial memory size $\alpha=0.99$ which is gradually annealed to 1 after $10^6$ iterations.

Figs \ref{comp_state} and \ref{comp_para} plot the posterior distributions of the state function and model parameters for the four methods.
For the three sampling approaches, we keep only the final $10^{6}$ samples and thin the samples every 1000 steps.
The uncertainty in the state function $x_2$ is much higher than $x_1$, as measurements are only available for $x_1$.
The results show a great agreement of the four methods.
Fig \ref{comp_para_conv} plots the convergence speeds of the parameter posterior distributions.
The upper half of the figure shows the results for the entire $2\times10^6$ iterations, while the lower half zooms in on the first $10^5$ iterations.
The plots indicate that NSVI converges faster than NSGLD and NPSGLD.
Remarkably, NPSGLD significantly improves the convergence speed compared to NSGLD.
Fig \ref{comp_x0} compares the posterior distribution of the initial state of the parameterized state path function $p(\hat{x}(0;w)\vert y)$ with the posterior distribution of the auxiliary initial state $p(x_0\vert y)$.
Since NSVI is an approximation method, the two distributions show slight differences.
In contrast, NSGLD and NPSGLD yield nearly identical distributions.

\begin{figure}[!htb]
  \centering
\includegraphics[scale=0.43]{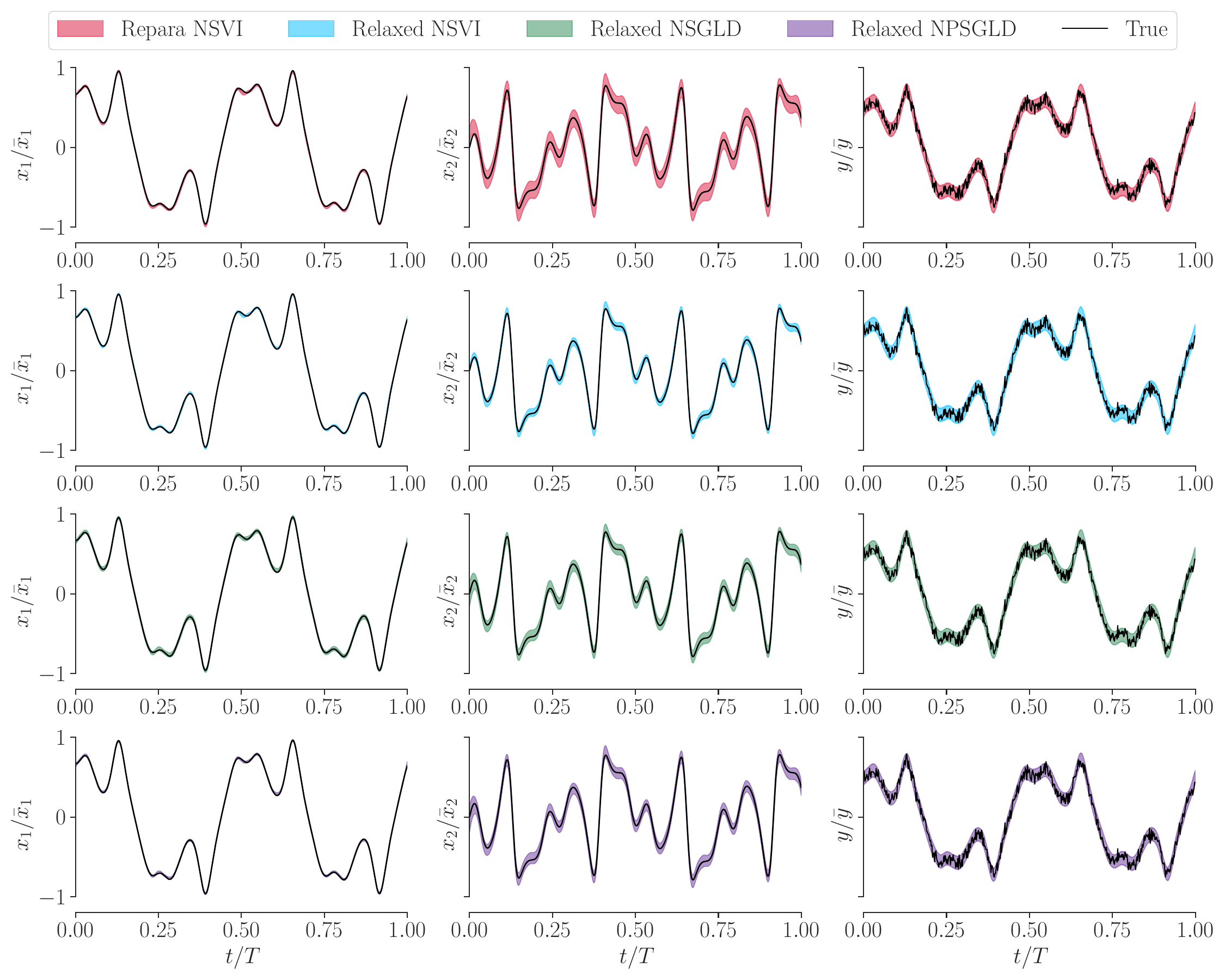}
  \caption{
  Example of section \ref{example 1}: posterior distributions of states and measurements.
  }
  \label{comp_state}
\end{figure}

\begin{figure}[!htb]
\centering
  \includegraphics[scale=0.43]{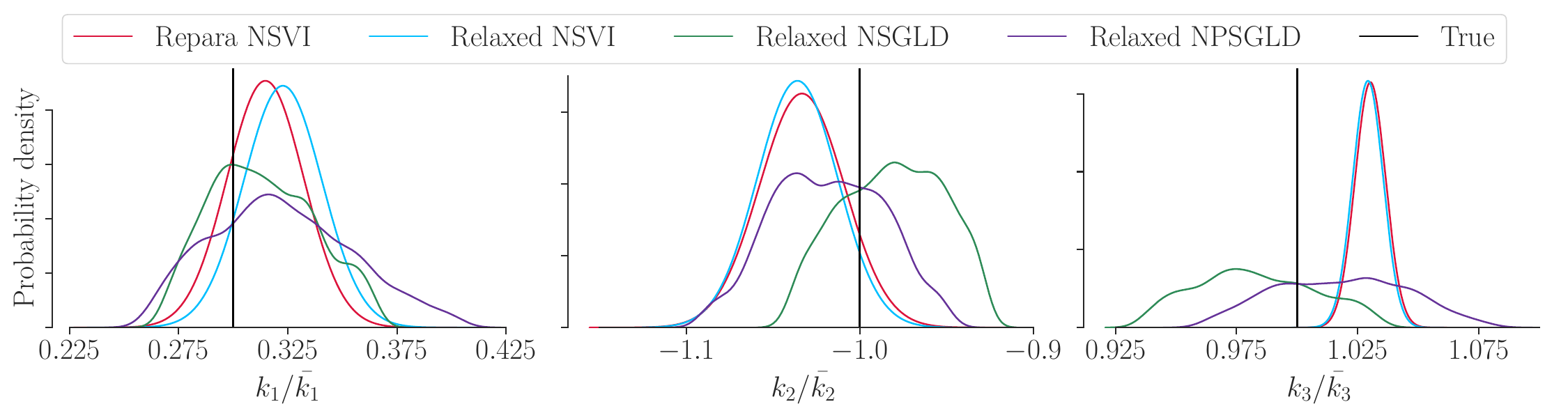}
  \caption{Example of section \ref{example 1}: posterior distributions of model parameters.}
  \label{comp_para}
\end{figure}

\begin{figure}[!htb]
\centering
  \includegraphics[scale=0.43]{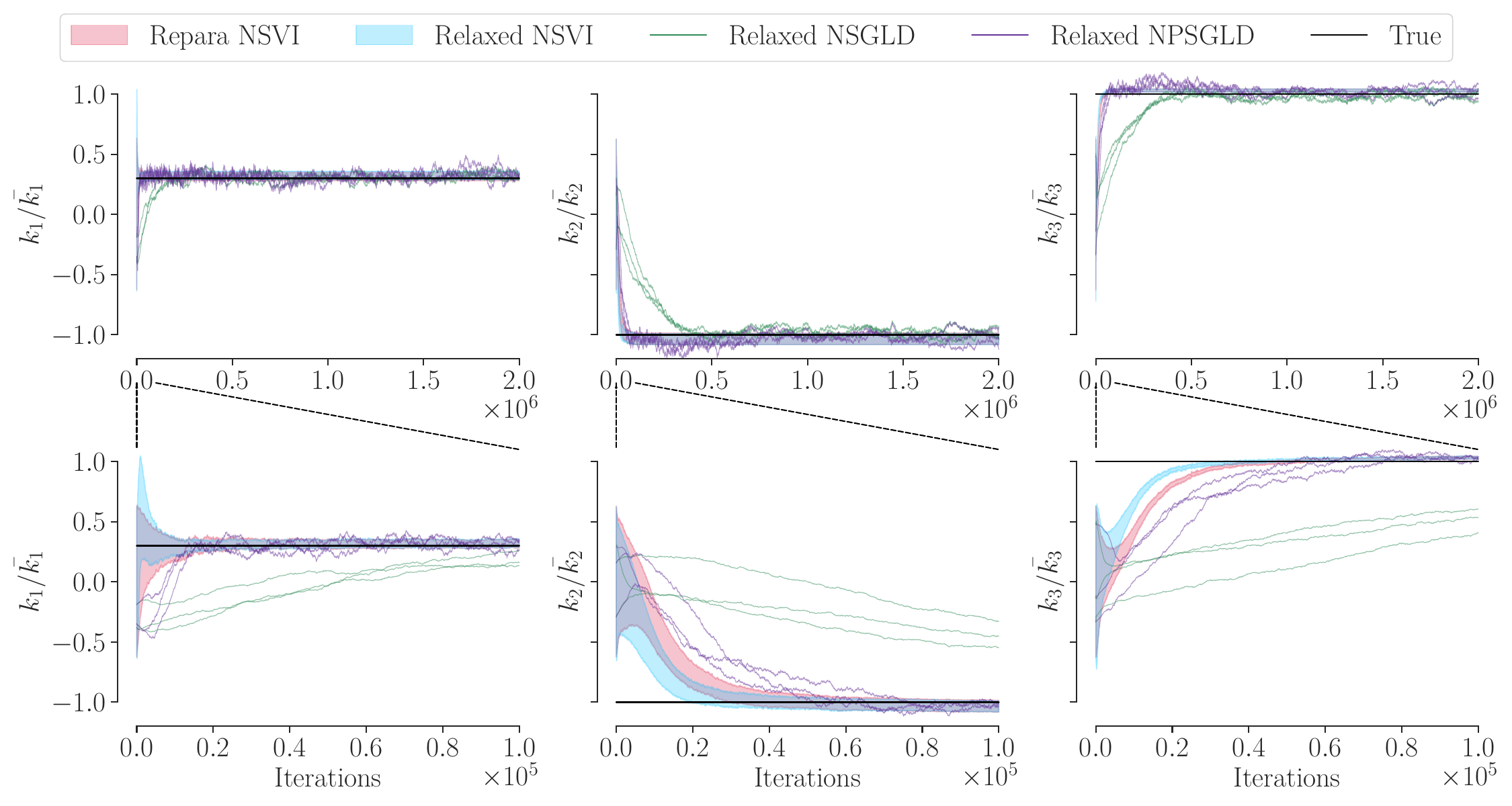}
  \caption{Example of section \ref{example 1}: convergence speeds of the model parameter posterior.}
  \label{comp_para_conv}
\end{figure}

\begin{figure}[!htb]
\centering
    \includegraphics[scale=0.43]{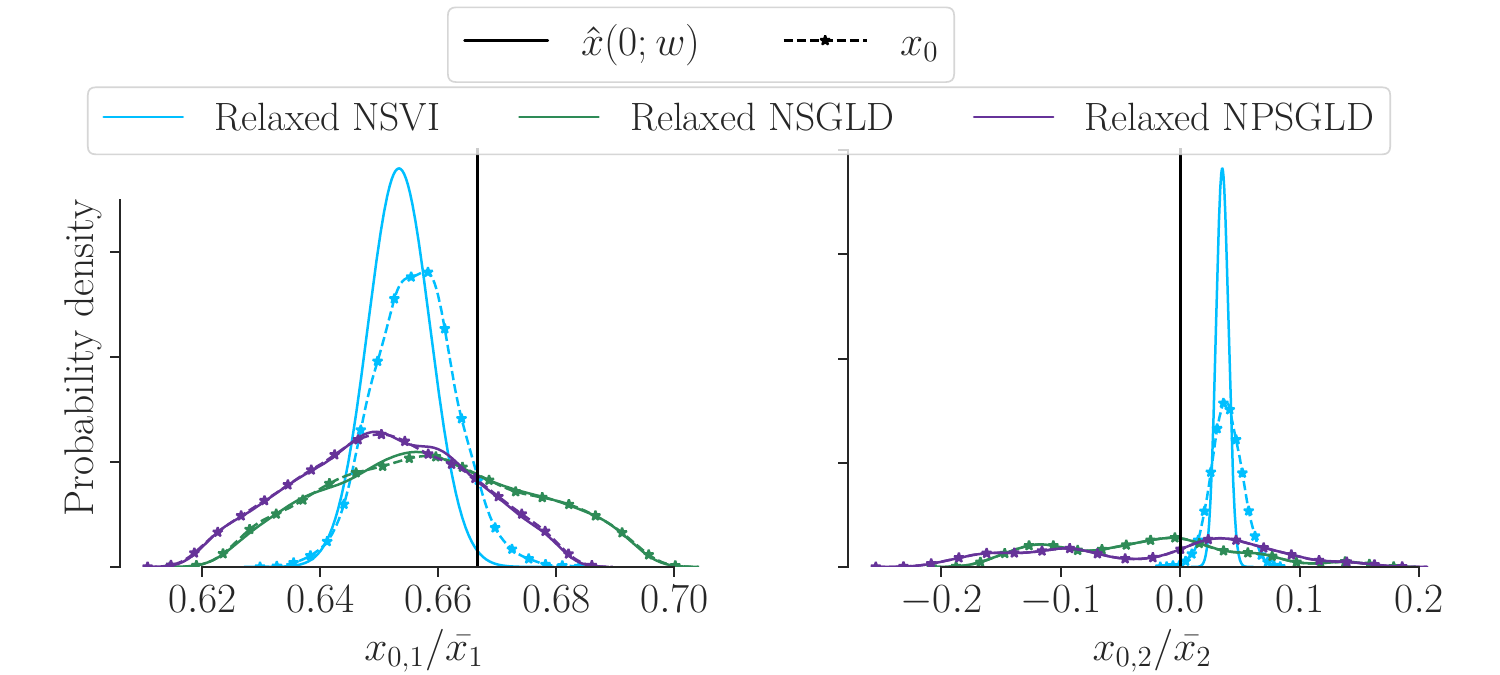}
  \caption{Example of section \ref{example 1}: posterior distributions of the state path initial state $\hat{x}(0;w)$ and the auxiliary initial state $x_0$.}
  \label{comp_x0}
\end{figure}

\subsection{Improvement using a residual function}
\label{section:2dof}

In this example, we investigate the improvement in state reconstruction and parameter estimation accuracy by including a residual function in neural networks.
We use a two-degree-of-freedom nonlinear system considered in \citep{kong2022non}.
The system consists of two masses: $m_1$ is connected to the wall through a linear damper and a linear-plus-cubic (Duffing) nonlinear spring, and $m_2$ is connected to $m_1$ via a linear spring and a linear-plus-cubic nonlinear damper.
We excite $m_1$ by a sinusoidal signal.
The absolute displacements of $m_1$ and $m_2$ are denoted by $y_1$ and $y_2$, respectively.
We define the new displacement variables $q_1 = y_1$ and $q_2 = y_2-y_1$ to write the governing equations:
\begin{align*}
    \begin{bmatrix}
        m_1, & 0 \\
        m_2, & m_2
    \end{bmatrix}
    \begin{bmatrix}
        \Ddot{q}_1\\
        \Ddot{q}_2
    \end{bmatrix}
    +
    \begin{bmatrix}
        c_1, & -c_2\\
        0, & c_2
    \end{bmatrix}
    \begin{bmatrix}
        \dot{q}_1 \\
        \dot{q}_2
    \end{bmatrix}
    +
    \begin{bmatrix}
        k_1, & -k_2 \\
        0, & k_2
    \end{bmatrix}
    \begin{bmatrix}
        q_1\\
        q_2
    \end{bmatrix}
    +
    \begin{bmatrix}
        k_1\epsilon_1q_1^3-c_2 \epsilon_2 \dot{q}_2^3
        \\
        c_2\epsilon_2\dot{q}_2^3
    \end{bmatrix}
    =
    \begin{bmatrix}
        F_0\sin{(\omega_0 t)}
        \\
        0
    \end{bmatrix}.
\end{align*}

Let $x_1 = q_1$, $x_2 = \dot{q}_1$, $x_3 = q_2$ and $x_4=\dot{q}_2$, we obtain a state space model with four state variables.
We measure the displacements $q_1$ and $q_1+q_2$ of $m_1$ and $m_2$, respectively.

The reference true parameter values are $m_1=m_2=1$, $c_1=c_2=0.2$, $k_1=k_2=1$, and $\epsilon_1=\epsilon_2=0.2$.
We normalize all eight parameters and the four states by 1.
The normalization constants for the measurements are $(\bar{y}_1, \bar{y}_2)=(1, 2)$.

To generate synthetic measurement data, we ran a 50-second forward simulation using the Runge-Kutaa method with a time step of 0.1s.
The initial states are $(x_1(0), x_2(0), x_3(0), x_4(0))^T=(0, 0, 0.5, 0)^T$.
The measurement noise standard deviations are set to 5\% of the measurement normalization constant.

For the parameterized state path function, the linear basis function is the radial basis
\begin{align*}
\hat{x}(t;w)
=
\sum_{k=1}^{K_b}
w_k
e^{
-\frac{
(x-z_k)^2
}{
2\sigma_{k}
}
},
\end{align*}
where $z_k$ and $\sigma_k$ are location and scale parameters. 
Specifically, we set $K_b=20$ and $\sigma_k=0.05$.
The location parameters $z_k$ are evenly spaced from 0 to 1.
The residual function includes a Fourier encoding layer with $K=10$, one hidden layer of width 10, and the swish activation function~\citep{ramachandran2017searching}.
This parameterization is designed because the radial basis function alone is insufficient to approximate the ground truth state path, and we anticipate that the residual function can correct this unknown error.

We ran four cases, including with and without the residual function, each solved by NSVI or NPSGLD.
For NSVI, we ran 300,000 iterations and annealed the partition function in the ELBO over the first 200,000 iterations.
The remaining optimization settings are the same as in the previous example.
For NPSGLD, we ran 3 MCMC chains in parallel for 3,000,000 steps, burned the first 1,000,000 samples, and thinned the chains every 10,000 steps.
The other optimization settings remain unchanged from the previous setup.

Fig. \ref{postpredstate_2dof} shows
the comparison of the posterior and predictive distributions of the four states and two measurements using NSVI and NPSGLD.
The left column presents NSVI results, and the right column presents NPSGLD results.
Each subplot is further divided: the left side shows the posterior distribution, and the right side shows the posterior predictive distribution. 
 Within each subplot, we compare results with and without the residual function.
 It is evident that the radial basis function alone cannot accurately reconstruct the state path, but with the residual function included, both NSVI and NPSGLD successfully reconstruct the state path.
Fig. \ref{postpara_2dof} compares the posterior distributions of the eight model parameters using NSVI and NPSGLD.
With the residual function, the posterior distributions from both algorithms, shown in green and purple,  are close to the ground truth values.
However, without the residual function, both algorithms fail to identify correct parameter values, shown in red and blue.
We used the parameter posterior distributions to generate posterior predictive distributions of the four states and two measurements, as shown in Fig.~\ref{postpredstate_2dof}, where the benefit of including a residual function is evident.
Finally, Fig.~\ref{convergence_speed_2dof} shows the convergence speeds of the parameter estimates.
The convergence speed of NPSGLD is quite acceptable compared to NSVI, and the results from NSVI and NPSGLD with the residual function (shown in green and purple) show strong agreement.
Moreover, without the residual function, the posterior of model parameters becomes non-identifiable. 
NPSGLD is unstable and converges to multiple local modes of the posterior distribution.

\begin{figure}[!htb]
  \centering
\includegraphics[scale=0.43]{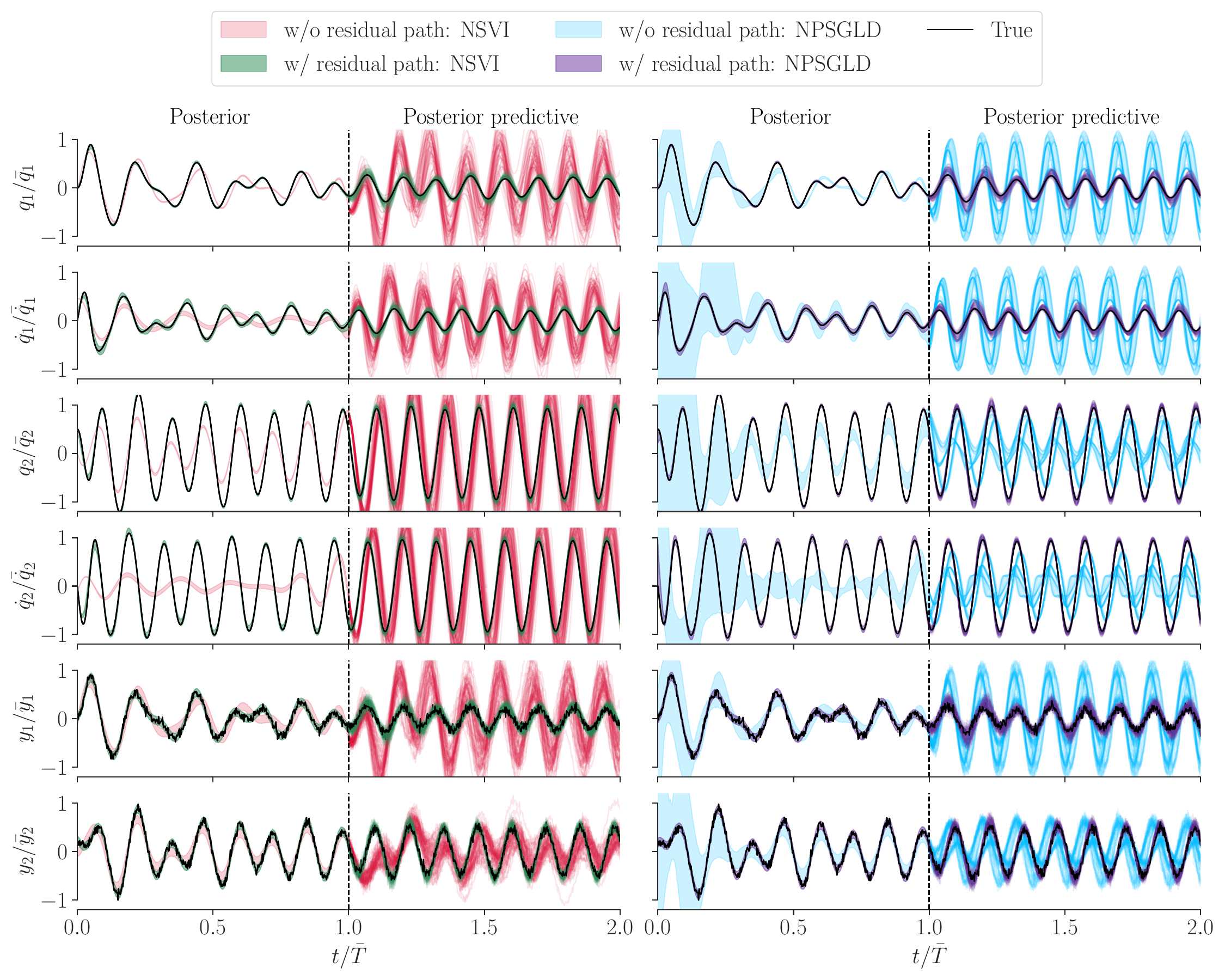}
  \caption{Example of section \ref{section:2dof}: posterior and predictive distributions of the states and measurements.}
  \label{postpredstate_2dof}
\end{figure}

\begin{figure}[!htb]
\centering
  \includegraphics[scale=0.43]{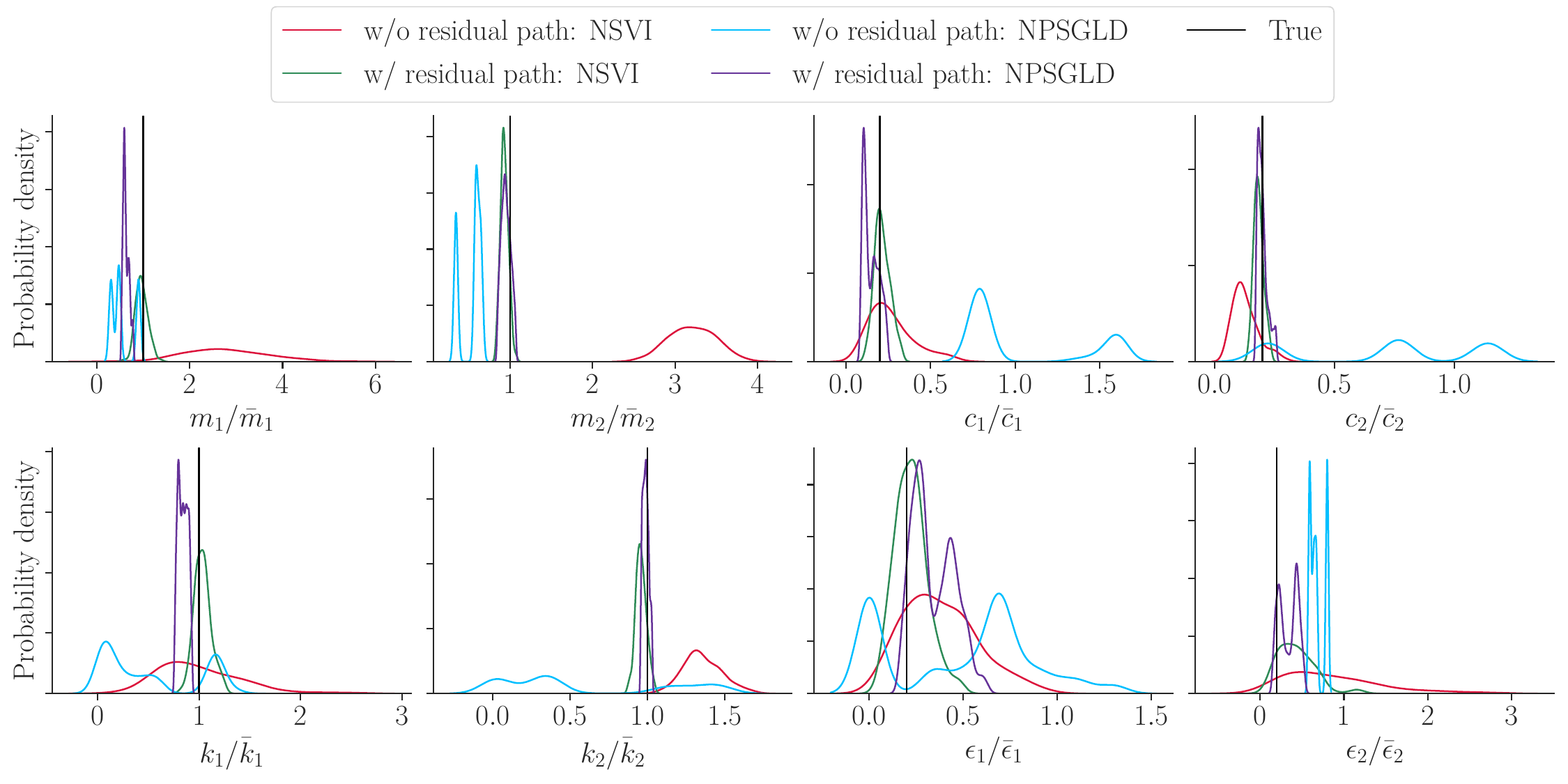}
  \caption{Example of section \ref{section:2dof}:  posterior distributions of model parameters.}  
  \label{postpara_2dof}
\end{figure}

\begin{figure}[!htb]
\centering
  \includegraphics[scale=0.43]{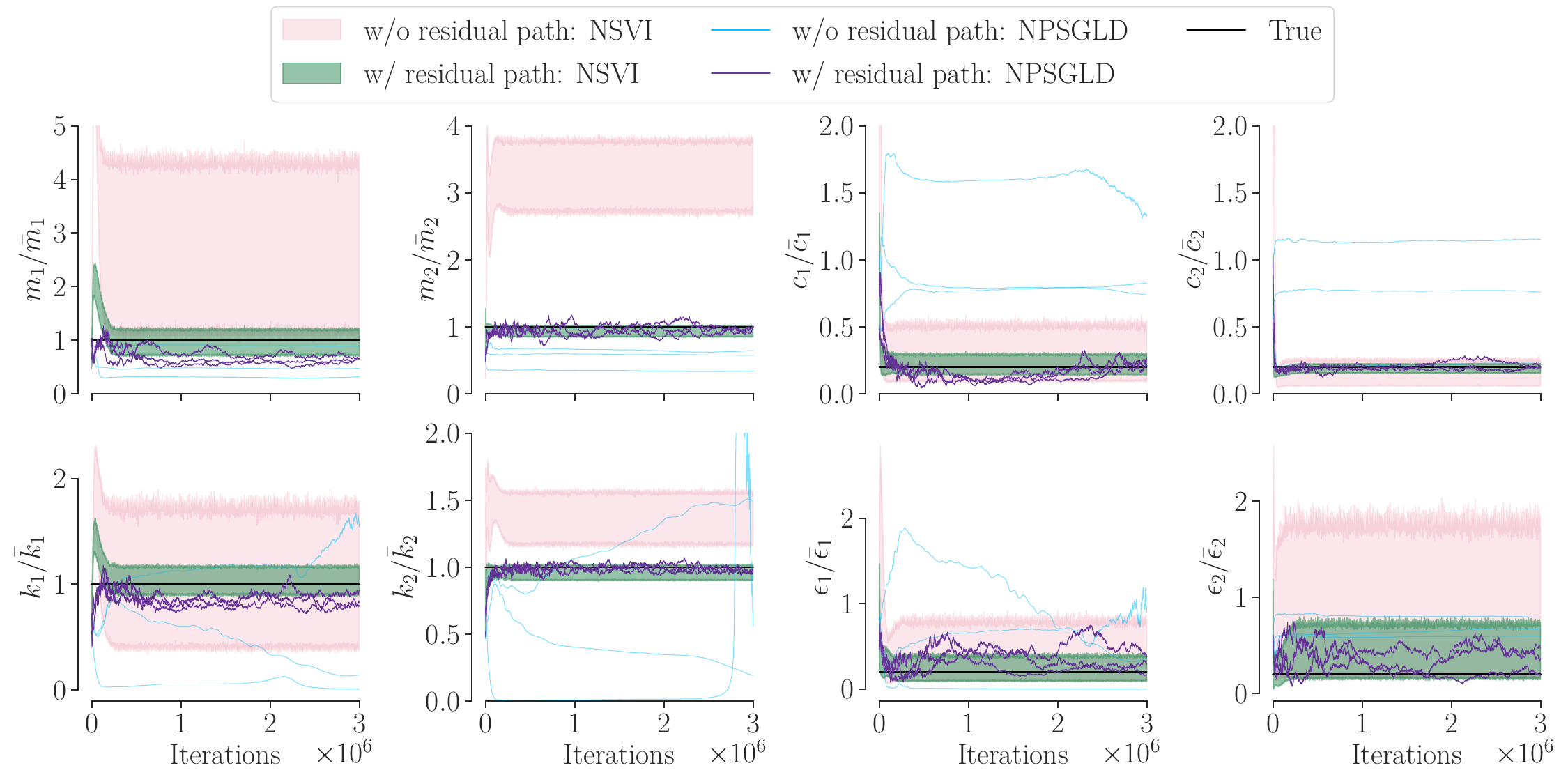}
  \caption{Example of section \ref{section:2dof}: convergence speeds of the model parameter posterior.}
  \label{convergence_speed_2dof}
\end{figure}

\subsection{High-dimensional problem: a twenty-story frame structure}
\label{section:high}

In this example, we study a twenty-story Bouc-Wen frame structure as a high-dimensional problem modified from~\citep{li2024multiple}.
The dynamic equations are
\begin{align*}
    \begin{bmatrix}
        m_1 & & &\\
        & m_2 & &\\
        & &\ddots &\\
        &&& m_{20}
    \end{bmatrix}
    \begin{bmatrix}
        a_1
        \\
        a_2
        \\
        \vdots
        \\
        a_{20}
    \end{bmatrix}
    +
    \begin{bmatrix}
        c_1+c_2 & -c_2 & &\\
        -c_2 & c_2+c_3 & -c_3 &\\
        &&\ddots &-c_{20}\\
        &&-c_{20}&c_{20}
    \end{bmatrix}
    \begin{bmatrix}
        v_1
        \\
        v_2
        \\
        \vdots
        \\
        v_{20}
    \end{bmatrix}
    \\
    +
    \begin{bmatrix}
        s_1+s_2 &-s_2 & &\\
        -s_2 & s_2+s_3 & -s_3 &\\
        &&\ddots &-s_{20}\\
        &&-s_{20}&s_{20}
    \end{bmatrix}
    \begin{bmatrix}
        z_1
        \\
        z_2
        \\
        \vdots
        \\
        z_{20}
    \end{bmatrix}
    =
    -
    \begin{bmatrix}
        m_1
        \\
        m_2
        \\
        \vdots
        \\
        m_{20}
    \end{bmatrix}
    a_g,
\end{align*}
where $a_{1:20}$ and $v_{1:20}$ are acceleration and velocity for each story, respectively.
The system input $a_g$ is the ground acceleration.
The hysteretic displacement $z_l$, for $l=2,\cdots, 20$, is given by
\begin{align*}
    \dot{z}_l
    =
    (v_l-v_{l-1})
    -
    \beta 
    \vert
    v_l-v_{l-1}
    \vert
    \vert
    z_l
    \vert^{n-1}z_l
    -
    \gamma(v_l-v_{l-1})\vert z_l \vert^n,
\end{align*}
and for $l=1$ is given by
\begin{align*}
    \dot{z}_1
    =
    v_1 
    - 
    \beta 
    \vert
    v_1
    \vert
    \vert
    z_1
    \vert^{n-1}z_1
    -
    \gamma v_1\vert z_1 \vert^n.
\end{align*}
The Bouc-Wen parameters are $\beta$, $\gamma$ and $n$.

The acceleration of each story is the measurement:
\begin{equation*}
    \begin{bmatrix}
        a_1\\
        \vdots
        \\
        a_{20}
    \end{bmatrix}
    =
    -
    \begin{bmatrix}
        a_{g}\\
        \vdots
        \\
        a_{g}
    \end{bmatrix}
    -
    \begin{bmatrix}
        \frac{1}{m_1} & &
        \\
        &\ddots &
        \\
        &&\frac{1}{m_{20}}
    \end{bmatrix}
    \left(
    \begin{aligned}
    &+
    \begin{bmatrix}
        c_1+c_2 &-c_2 & &\\
        -c_2 & c_2+c_3 & -c_3 &\\
        &&\ddots &-c_{20}\\
        &&-c_{20}&c_{20}
    \end{bmatrix}
    \begin{bmatrix}
        v_1
        \\
        v_2
        \\
        \vdots
        \\
        v_{20}
    \end{bmatrix}
        \\
        &+
        \begin{bmatrix}
        s_1+s_2 & -s_2& &\\
        -s_2 & s_2+s_3 & -s_3 &\\
        &&\ddots &-s_{20}\\
        &&-s_{20}&s_{20}
    \end{bmatrix}
    \begin{bmatrix}
        z_1
        \\
        z_2
        \\
        \vdots
        \\
        z_{20}
    \end{bmatrix}
    \end{aligned}
    \right).
\end{equation*}

We set $m_{1:20}=1$ kg, $c_{1:20}=0.25$ Ns/m, and randomly sampled twenty stiffness values $s_{1:20}$ from a uniform distribution $\mathcal{U}([8, 10])$.
We assume $m_{1:20}$ and $c_{1:20}$ are known and only estimate $s_{1:20}$. 
To generate a synthetic measurement dataset, we used the first 5-second El-Centro NS earthquake signal 
\url{https://www.vibrationdata.com/elcentro.htm} as the ground acceleration $a_g$. 
We corrupted the measurement data $a_{1:20}$ by zero-mean Gaussian noises whose standard deviations were set to 1\% of the root-mean-square values of the story accelerations.

The state path is parameterized by 100 evenly spaced radial bases with the same length scale of 0.01.
The residual function includes a Fourier encoding layer with $K=10$, one hidden layer of width 10, and the swish activation function.

We ran NSVI 50,000 steps and NPSGLD 200,000 steps with 5 chains, respectively.
The setups for both algorithms are similar to those in the previous section.
Figs. \ref{high v}-\ref{high theta} plot the posterior distributions of velocities, hysteretic displacements, accelerations, and stiffnesses.
Both algorithms successfully reconstruct the 40 states.
For the stiffness parameters, NPSGLD identifies all $s_{1:20}$ with negligible errors, while NSVI accurately identifies $s_1$ and $s_{3:20}$, with a small error in $s_2$.
We used the full 30-second ground acceleration data to check the posterior predictive distribution.
Fig.~\ref{high pred} plots the result for story frames 1, 5, 10, 15, and 20.
The predictions are highly accurate, although $z_{15}$ and $z_{20}$ show minor errors.

\begin{figure}[!htp]
\centering
  \includegraphics[scale=0.43]{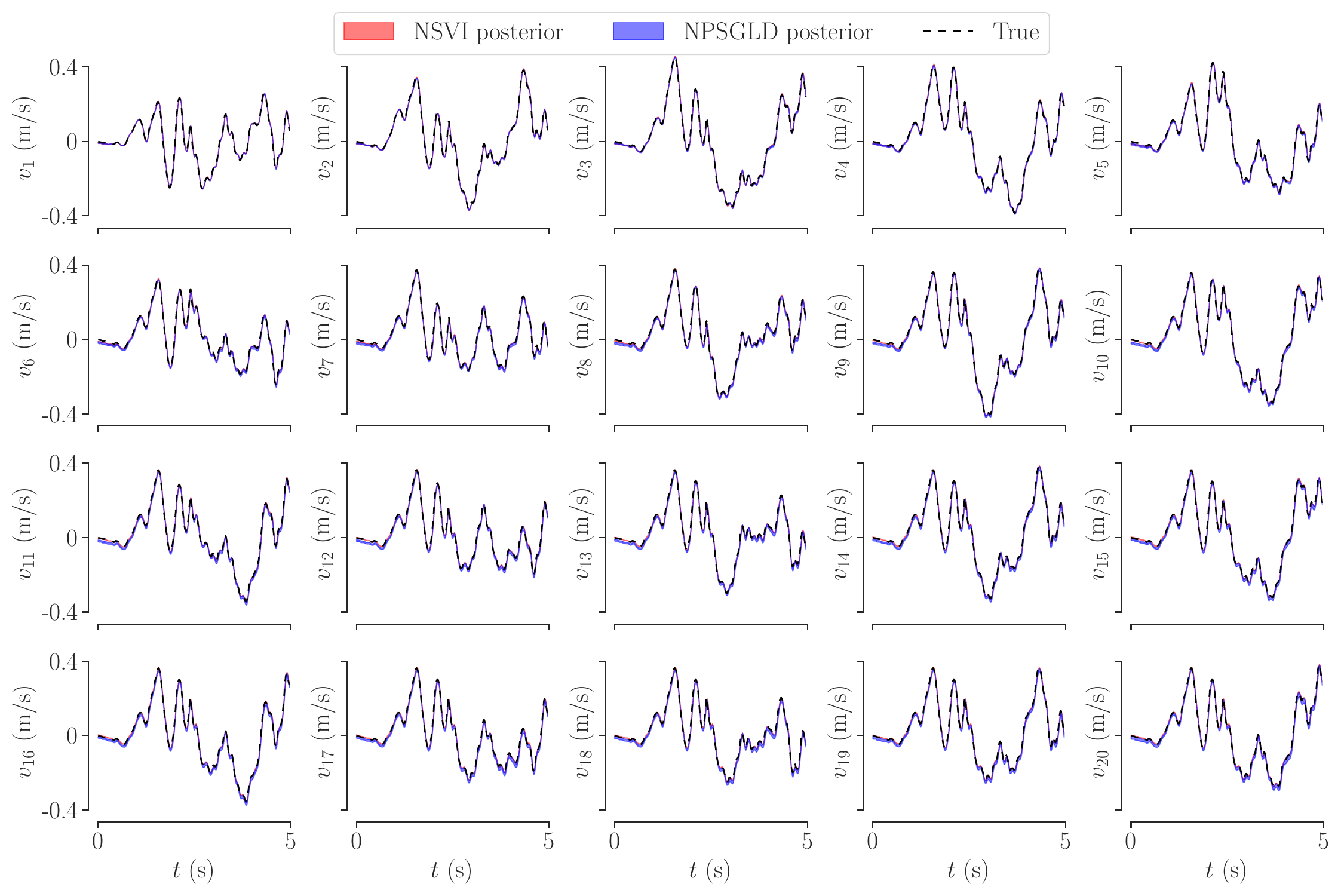}
  \caption{Example of section~\ref{section:high}:
  NSVI and NPSGLD posterior distributions of velocities.}
  \label{high v}

  \includegraphics[scale=0.43]{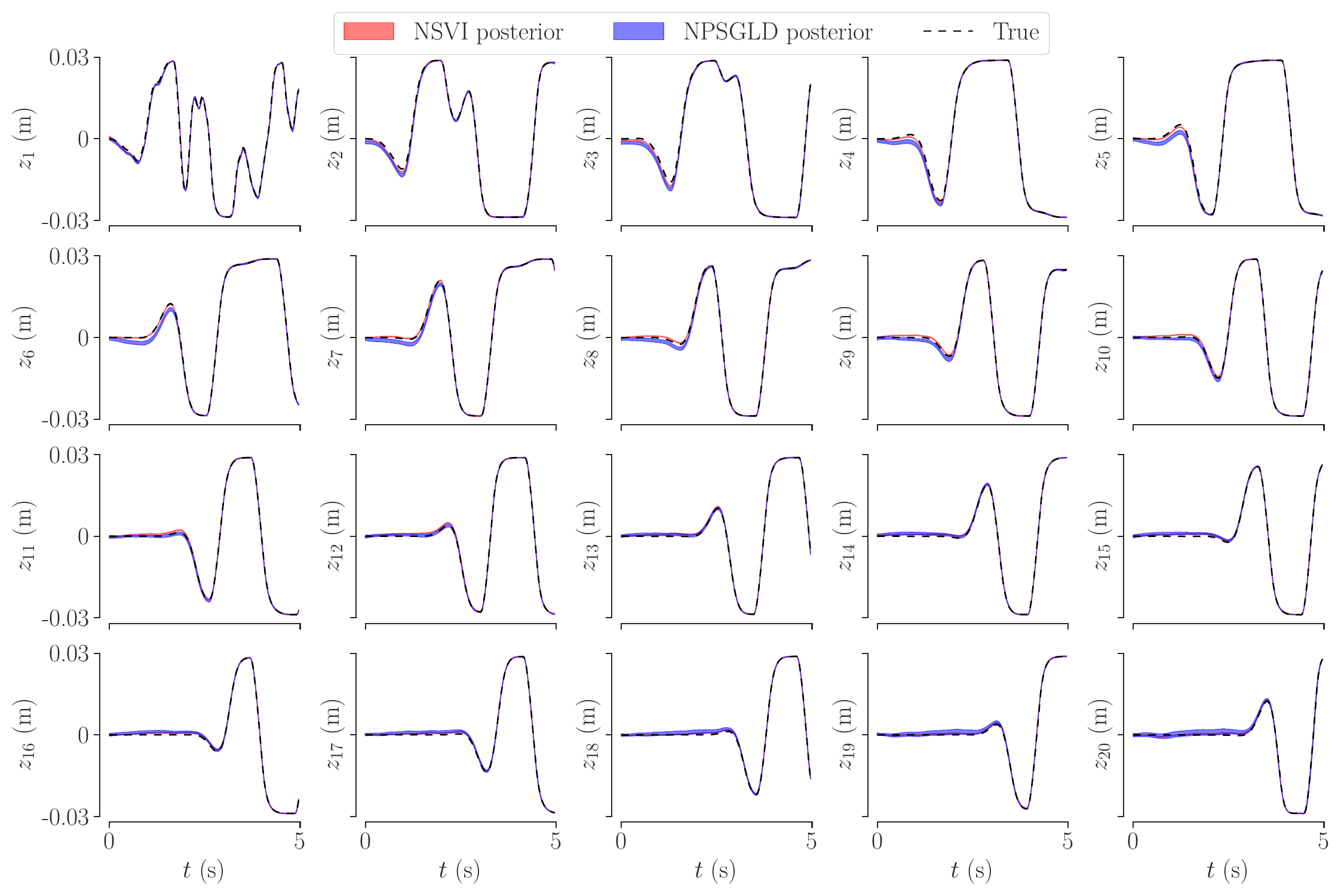}
  \caption{Example of section~\ref{section:high}:
  NSVI and NPSGLD posterior distributions of hysteretic displacements.}
  \label{high z}
\end{figure}

\begin{figure}[!htp]
\centering
  \includegraphics[scale=0.43]{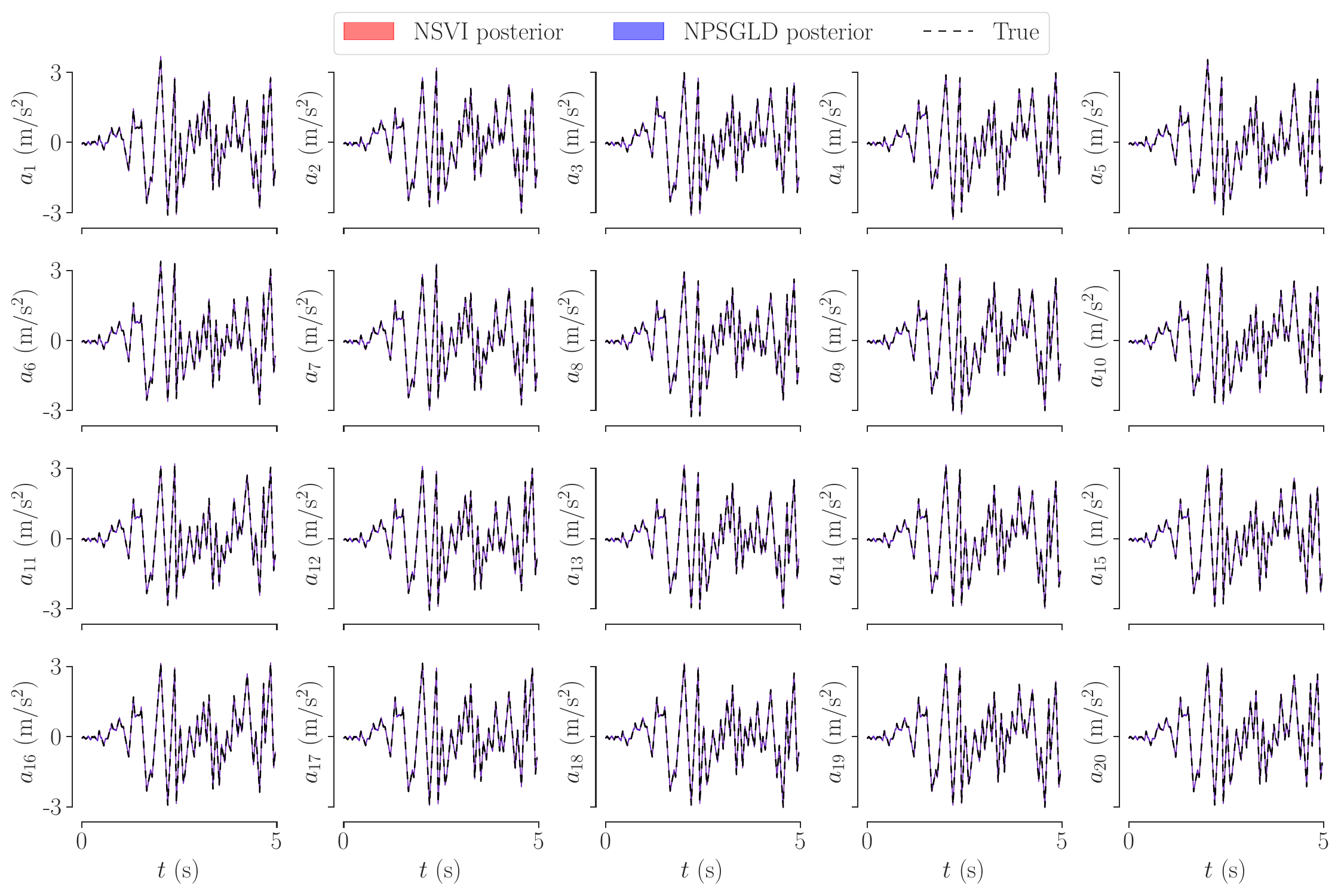}
  \caption{Example of section~\ref{section:high}: 
  NSVI and NPSGLD posterior distributions of acceleration measurements.}
  \label{high y}

  \includegraphics[scale=0.43]{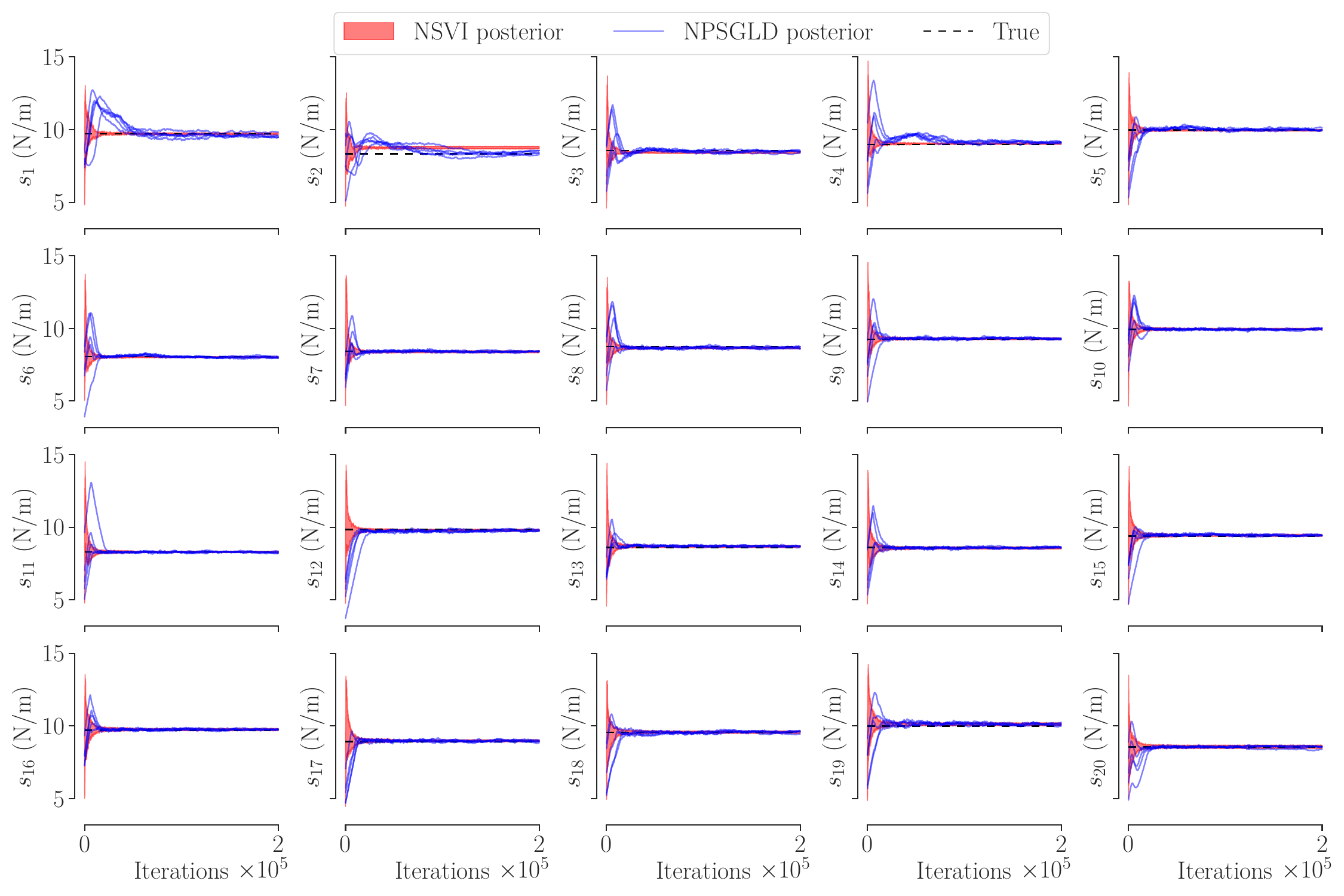}
  \caption{Example of section~\ref{section:high}: 
  NSVI and NPSGLD 
  convergence speeds of the model parameter posterior.}
  \label{high theta}
\end{figure}

\begin{figure}[!htp]
\centering
  \includegraphics[scale=0.43]{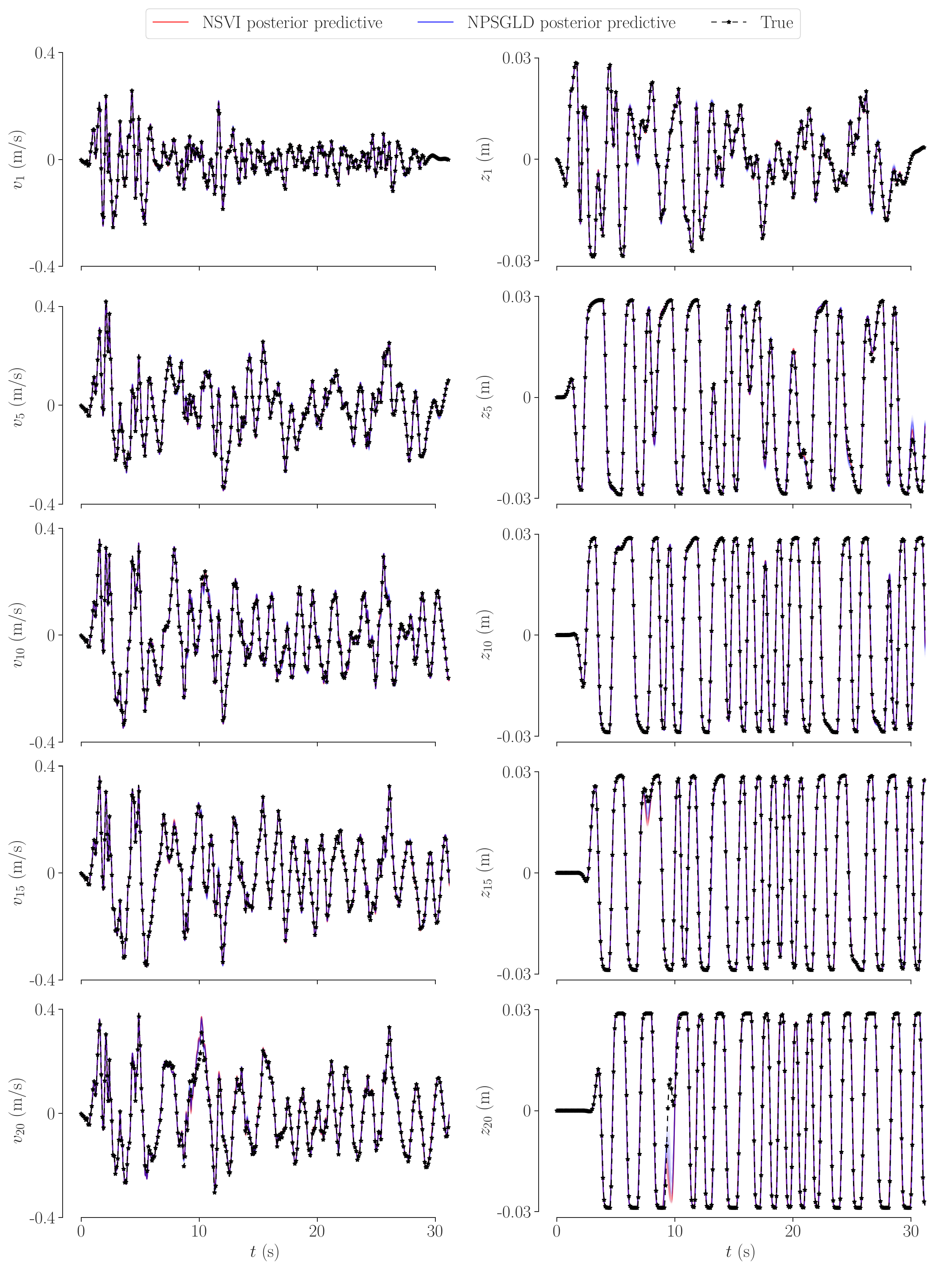}
  \caption{Example of section~\ref{section:high}: 
  NSVI and NPSGLD posterior predictive distributions.}
  \label{high pred}
\end{figure}

\subsection{Experimental example: nonlinear energy sink device}
\label{section:nes}

Last, we validate NIFF using an experimental example with a nonlinear energy sink device. 
The experimental details and data have been published in~\citep{silva2019evaluation}.
The nonlinear energy sink device is a Duffing-type oscillator, which is designed to transfer and dissipate energy. 
\citet{lund2021variational} used the unscented Kalman filter~\citep{wan2000unscented} to identify a mathematical model for this device from experimental data.
The model is given by
\begin{align}
\label{nes}
    m\Ddot{x}+c_{\nu}\dot{x}+c_f \tanh{(200\dot{x})}+kx +zx^3=
    -m \Ddot{x}_g,
\end{align}
where $\tanh{(200\dot{x})}$ is a differentiable approximation of Coulomb damping $\text{sign}\dot{x}$, and $\Ddot{x}_g$ is the excitation signal.
The mass $m$ is known to be 0.664 kg, and
the four parameters $c_{\nu}$, $c_f$, $k$ and $z$ need to be identified.
For more information on the experimental setup and dataset, refer to section 2 and Table 1 in \citep{lund2020identification}.

We follow approach B, proposed by the authors, to process two datasets simultaneously.
This is achieved by stacking two independent governing equations using Eq.~(\ref{nes}), where each governing equation corresponds to one dataset.
We denote the displacement and velocity of the nonlinear energy sink device in the first experiment by $x_1$ and $x_2$, and in the second experiment by $x_3$ and $x_4$.
The two experiments share the same four parameters, so the four-dimensional state space model is:
\begin{align*}
    \dot{x}_1 
    &= x_2,
    \quad
    \dot{x}_2 
    =
     -
     \frac{1}{m}
     \left(
     c_{\nu}x_2+c_f \tanh{(200 x_2)}+kx_1 +zx_1^3
     \right)
    - \Ddot{x}_{g, 1},
    \\
    \dot{x}_3 
    &= x_4,
    \quad
    \dot{x}_4 
    =
     -
     \frac{1}{m}
     \left(
     c_{\nu}x_4+c_f \tanh{(200 x_4)}+kx_3 +zx_3^3
     \right)
    - \Ddot{x}_{g, 2}.
\end{align*}
The two measurements are the displacement and relative acceleration of the nonlinear energy sink device:
\begin{align*}
    y_1 
    &= x_1,
    \quad
    y_2 =
    -
     \frac{1}{m}
     \left(
     c_{\nu}x_2+c_f \tanh{(200 x_2)}+kx_1 +zx_1^3
     \right),
     \\
     y_3 &= x_3,
     \quad
     y_4 =
      -
     \frac{1}{m}
     \left(
     c_{\nu}x_4+c_f \tanh{(200 x_4)}+kx_3 +zx_3^3
     \right).
\end{align*}

Since the measurement data were collected at 4096 Hz, we subsample only at these discrete time points when running NIFF to evaluate the physics-informed conditional prior, rather than sampling time uniformly as in \qref{logpriortarget}.

The two training experimental datasets are shown in Fig.~\ref{nes y posterior} in blue.
The left column displays the entire 90-second data.
The top two rows show the displacement and acceleration of the nonlinear energy sink under a sine sweep excitation signal with a 5 Hz frequency and a 2.7 mm maximum amplitude.
The bottom two rows show the displacements and accelerations under a sine excitation signal with a 5 Hz frequency and a variable amplitude that increases and then decreases, peaking at 2.7 mm.
Due to the rapid signal variations, we use 4000 radial basis terms for each signal without employing a residual function.
We present only the NSVI result, as NPSGLD converges extremely slowly, i.e., several hours, for this example due to the large number of unknown parameters. 

Fig.~\ref{nes y posterior} plots the posterior distributions of displacements and accelerations.
Since the details of the entire signals in the left column are not visible, we present zoomed-in regions in the right column.
The posterior distribution aligns well with the measurement data.
Fig.~\ref{nes para post} plots the posterior distribution of the four model parameters in the bottom row and their convergence speeds in the top row.
Our identified parameter values are similar to the values reported in Tables 2, 3, and 4 of \citep{lund2020identification}.
For example, in Table 4, the author reported one set of estimated values 
$
c_{\nu}=0.344\ \text{Ns/m},
c_f = 0.064\ \text{N},
k=33.1\ \text{Ns/m},
$
and 
$z = 6.54\times 10^5 \ \text{N/}\text{m}^3$.
Finally, we evaluate the posterior predictive distributions using four different experimental datasets.
Fig.~\ref{nes postpred} summarizes the prediction results, with each row corresponding to one experimental dataset.
The first column shows displacements, and the second column shows accelerations.
To account for randomness in the posterior predictions, we apply a modified version of the normalized mean square error (MSE) indicator~\citep{worden1990data} used in \citep{lund2020identification}:
\begin{align*}
    MSE =
    \E{\hat{X}_{1:N}, \hat{\Ddot{X}}_{1:N}}
    {
    \frac{100}{N}
    \sum_{i=1}^N
    \left(
    \frac{
    (x_i-\hat{X}_i)^2
    }{
    \sigma^2_d
    }
    +
    \frac{
    (\Ddot{x}_i-\hat{\Ddot{X}}_i)^2
    }{
    \sigma^2_a
    }
    \right)
    },
\end{align*}
where $\hat{X}_{1:N}$ and $\hat{\Ddot{X}}_{1:N}$ are the posterior predictive displacements and accelerations at $N$ sampling times,
$x_i$ and $\Ddot{x}_i$ are the measurement data, and $\sigma_d^2=1.44\times 10^{-12}$ m$^2$ and $\sigma_a^2=5.32\times 10^{-3}$ (m/s$^2$)$^2$ are the measurement variances chosen from \citep{lund2020identification}.
Table.~\ref{mse} summarizes the normalized MSE values for the four validation experimental datasets.
These values are similar in magnitude to those in Tables 2, 3, and 4 of~\citep{lund2020identification}, where reported normalized MSE values range from $10^8$ to $10^9$.

\begin{figure}[!htbp]
\centering
  \includegraphics[scale=0.43]{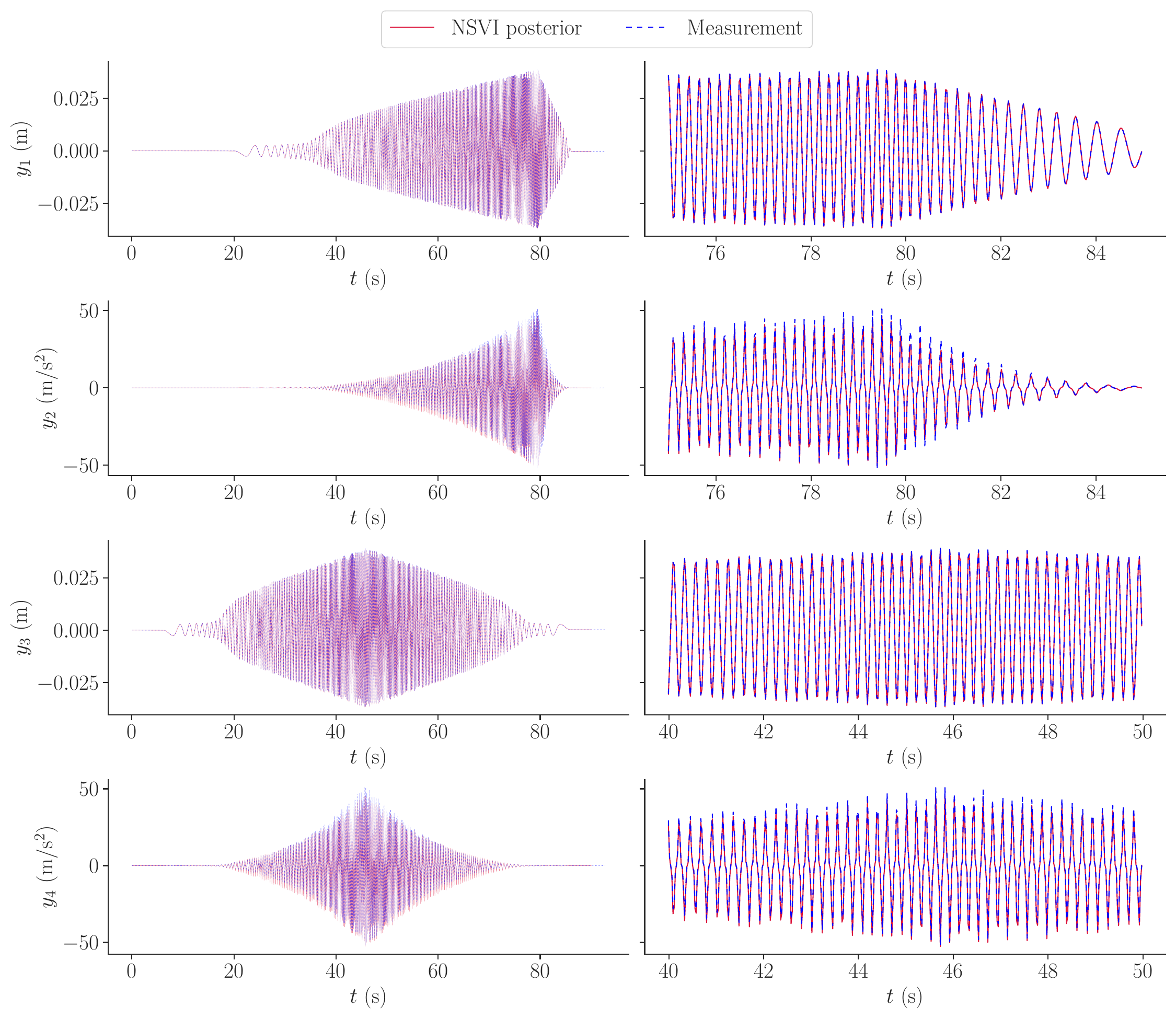}
  \caption{Example of section~\ref{section:nes}: NSVI posterior distribution. The left column includes entire dataset results, and the right column is a zoomed-in view.}
  \label{nes y posterior}
  \includegraphics[scale=0.43]{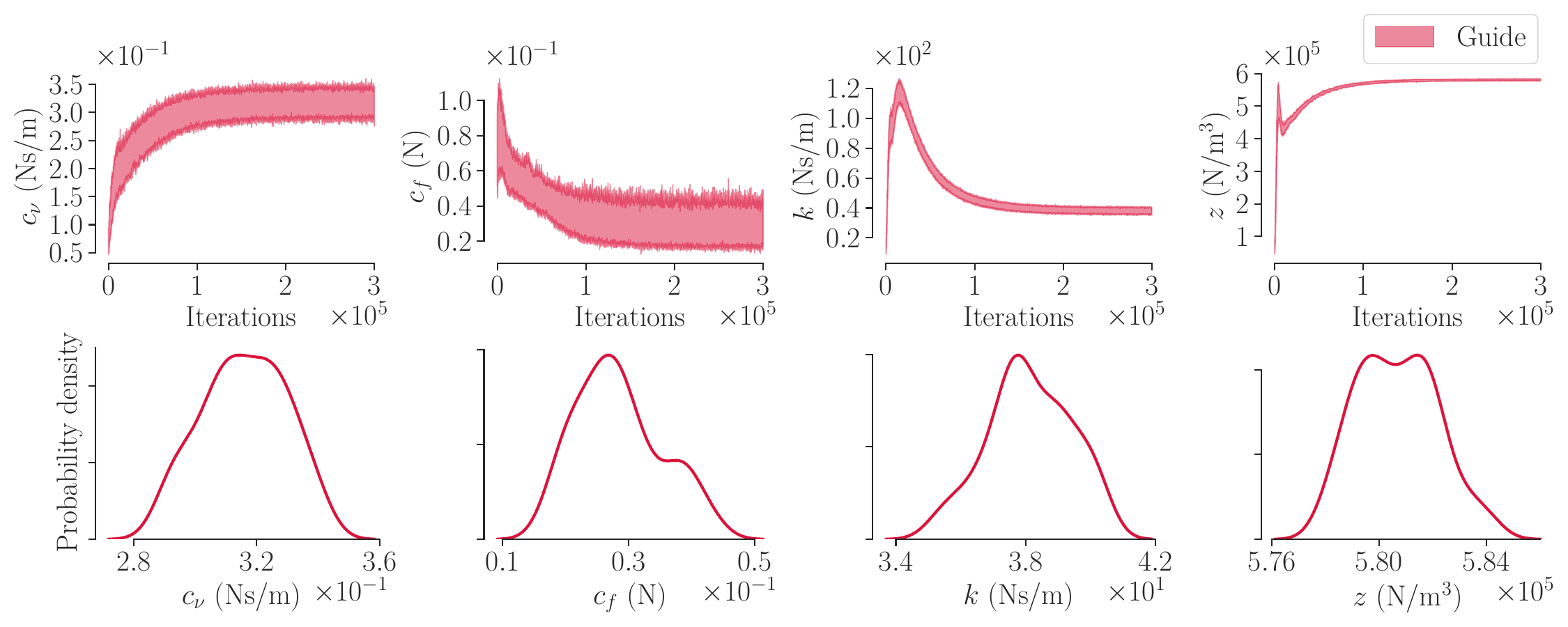}
  \caption{Example of section~\ref{section:nes}: NSVI posterior distribution of model parameters and the convergence speed.}
  \label{nes para post}
\end{figure}

\begin{figure}[!h]
\centering
  \includegraphics[scale=0.43]{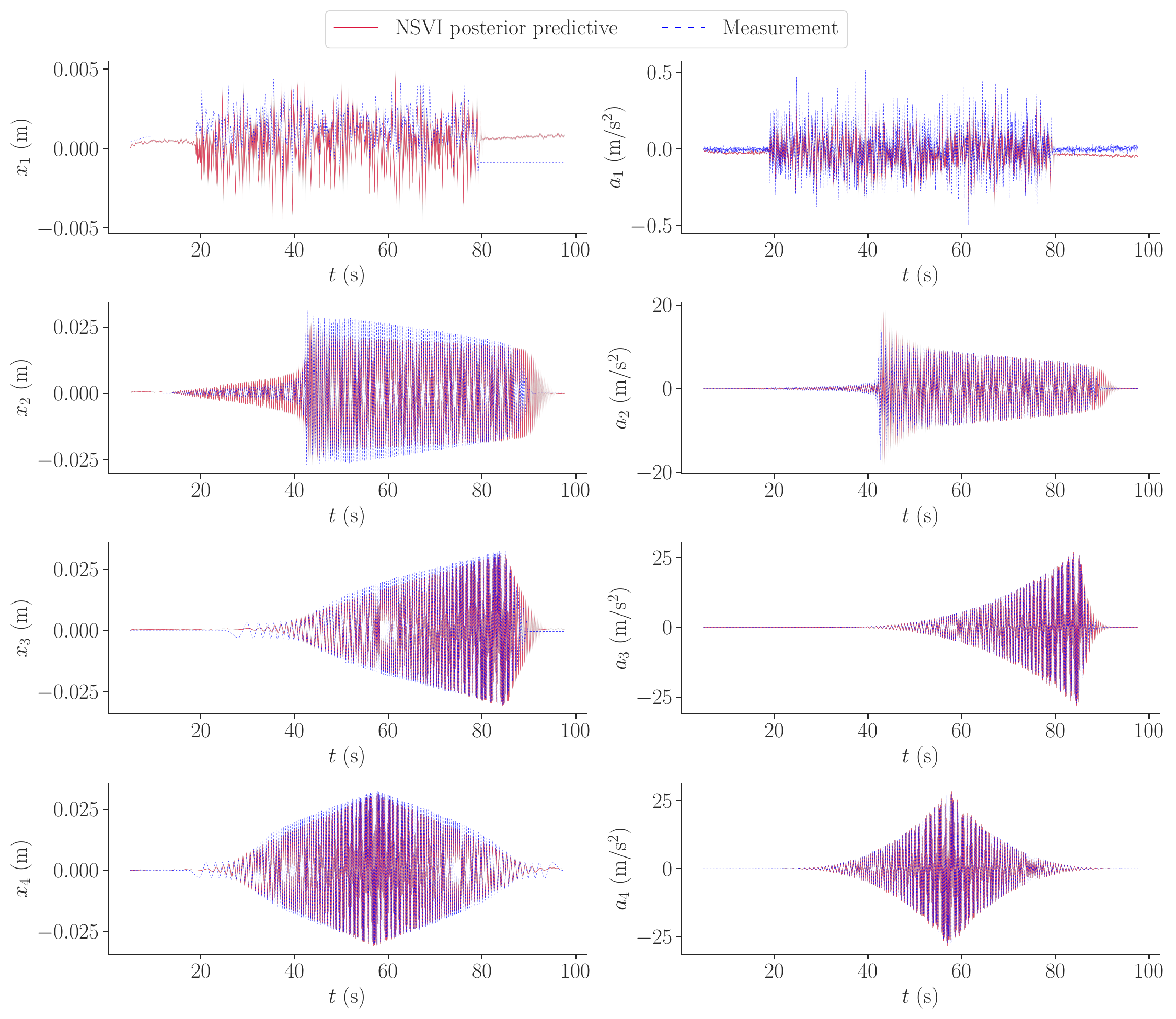}
  \caption{Example of section~\ref{section:nes}: NSVI posterior predictive distribution.}
  \label{nes postpred}
\end{figure}

\begin{table}[!htbp]
\caption{MSE for posterior predictive validation.}
\centering
\label{mse}
\begin{tabular}{ccccc}
\toprule[0.5mm]
& 
Signal 1
&
Signal 2
& 
Signal 3
& 
Signal 4
\\
\midrule[0.25mm]
MSE/10$^8$
&
1.6
&
16.9
&
8.1
&
1.7
\\
\bottomrule[0.5mm]
\end{tabular}
\end{table}
\section{Conclusions}
\label{sec: conclusions}

We have developed the neural information field filter for Bayesian estimation of states and parameters in dynamical systems.
NIFF improves parameterization expressiveness and reconstruction accuracy by representing the dynamical state path with a residual neural network.
To this end, we introduced a generalized physics-informed conditional prior, incorporating a kernel information Hamiltonian that measures the similarity between the initial state of the parameterized state path and an auxiliary initial state.
We showed that the physics-informed conditional prior defined in \citep{hao2024information} is a special case of this generalized physics-informed conditional prior when a Dirac kernel is used.
To sample from the posterior distribution, we developed an optimization algorithm, nested stochastic variational inference, and a sampling algorithm, nested preconditioned stochastic gradient Langevin dynamics.
We conducted three synthetic examples and one experimental example.
In the first example, using a single-degree-of-freedom Duffing oscillator, we verified that our non-reparameterized state path function approach produced similar results to the reparameterized approach~\citep{hao2024information}.
In the second example, using a two-degree-of-freedom nonlinear system~\citep{kong2022non}, we demonstrated that adding a residual function significantly improved reconstruction accuracy.
In the third synthetic example, we tested NIFF’s performance on a high-dimensional, twenty-story frame structure model, with both numerical algorithms yielding accurate results.
Last, we successfully validated NIFF using a nonlinear energy sink experimental example.
In summary, NIFF provides a powerful and flexible framework for Bayesian estimation in dynamical systems.

\section*{Acknowledgments}
This work was supported by a Space Technology Research Institutes Grant (number 80NSSC19K1076) from NASA’s Space Technology Research Grants Program. 

\section*{Declaration of generative AI and AI-assisted technologies in the writing process}
During the preparation of this work the authors used Grammarly and Chat GPT in order to correct spelling, grammatical, and syntactical errors. After using this tool/service, the authors reviewed and edited the content as needed and take full responsibility for the content of the publication.

\appendix
\section{Proof of Proposition~\ref{prop1}}
\label{appendix:proof1}
\begin{proof}
First we show $p(w_{-i}\vert x_0, \theta)=
\Tilde{p}(w_{-i}\vert x_0, \theta)$.
From the definition of $p(w\vert x_0, \theta)$, we have
    \begin{align*}
        p(w\vert x_0, \theta)
        &=
        \dfrac{
            e^{
            -\beta H(w,  \theta)
            -H(\hat{x}(0;w), x_0)
        }
        }{
        \int dw^{\prime}\
        e^{
            -\beta H(w^{\prime},  \theta)
            -H(\hat{x}(0;w^{\prime}), x_0)
            }
        }
        \\
        &=
        \dfrac{
            \delta(
                w_i-\mathcal{T}(x_0;w_{-i})
            )
            e^{
                -\beta H(w, \theta)
            }
        }{
        \int dw^{\prime}_{-i}
        \int dw^{\prime}_i\
        \delta(
        w_i^{\prime}-\mathcal{T}(x_0;w_{-i}^{\prime})
        )
        e^{
            -\beta H(w^{\prime},  \theta)
        }
        }
        \\
        &=
        \dfrac{
            \delta(
                w_i-\mathcal{T}(x_0;w_{-i})
            )
            e^{
            -\beta H(w, \theta)
            }
        }{
        \int dw^{\prime}_{-i}\
        e^{
            -\beta H(
            w^{\prime}_{-i},
            \mathcal{T}(x_0; w^{\prime}_{-i}),
            \theta)
            )
        }
        }.
    \end{align*}
    Then, we have
    \begin{align*}
        p(w_{-i}\vert x_0, \theta)
        &=
        \int dw_i \ p(w\vert x_0, \theta)
        \\
        &=
        \int 
        dw_i\
        \frac{
            \delta(
                w_i-\mathcal{T}(x_0;w_{-i})
            )
            e^{
            -\beta H(w, \theta)
            }
        }{
        \int dw^{\prime}_{-i}\
        e^{
            -\beta H(  
            w^{\prime}_{-i},
            \mathcal{T}(x_0; w^{\prime}_{-i}),
            \theta)
        }
        }
        \\
        &=
        \frac{
            e^{
            -\beta H(w_{-i}, \mathcal{T}(x_0; w_{-i}) ,\theta)
            }
        }{
        \int dw^{\prime}_{-i}\
        e^{
            -\beta H(  
            w^{\prime}_{-i},
            \mathcal{T}(x_0; w^{\prime}_{-i}),
            \theta)
        }
        }
        \\
        &=
        \frac{
            e^{
            -\beta \Tilde{H}(w_{-i}, x_0, \theta)
            }
        }{
        \int dw^{\prime}_{-i}\
        e^{
            -\beta \Tilde{H}(w_{-i}^{\prime}, x_0, \theta)
        }
        }
        \\
        &=
        \Tilde{p}_(w_{-i}\vert x_0, \theta).
    \end{align*}

    Next, we show $p(w\vert \theta)=\Tilde{p}(w\vert \theta)$.
    We need to use the composition with a function property of the Dirac function~\citep{gel2014generalized}, and derive
    \begin{align*}
        \delta(
                w_i-\mathcal{T}(x_0;w_{-i})
            )
        =
        \frac{
            \delta(x_0-
                \mathcal{T}^{-1}(w_i; w_{-i})
            )
        }{
            \abs{
                \frac{
                d\mathcal{T}(x;w_{-i})
                }{
                    dx
                }\big\vert_{x=\mathcal{T}^{-1}(w_i; w_{-i})}
            }
        }.
    \end{align*}
    We first work on the generalized physics-informed conditional prior:
    \begin{align*}
        p(w\vert \theta)
        &=
         \int dx_0 \
        p(w\vert x_0, \theta)
        p(x_0)
        \\
        &=
        \int dx_0\
        \frac{
            \delta(
                w_i-\mathcal{T}(x_0;w_{-i})
            )
            e^{
            -\beta H(w, \theta)
            }
        }{
        \int dw^{\prime}_{-i}\
        e^{
            -\beta H(
            w^{\prime}_{-i},  
            \mathcal{T}(x_0; w^{\prime}_{-i}),
            \theta)
        }
        }
        p(x_0)
        \\
        &=
        \int dx_0\
        \frac{
            \delta(x_0-
                \mathcal{T}^{-1}(w_i; w_{-i})
            )
        }{
            \abs{
                \frac{
                d\mathcal{T}(x;w_{-i})
                }{
                    dx
                }\big\vert_{x=\mathcal{T}^{-1}(w_i; w_{-i})}
            }
        }
        \frac{    
            e^{
            -\beta H(w, \theta)
            }
        }{
        \int dw^{\prime}_{-i}\
       e^{
            -\beta H(  
            w^{\prime}_{-i},
            \mathcal{T}(x_0; w^{\prime}_{-i}),
            \theta)
        }
        }
        p(x_0)
        \\
        &=
        \frac{1}{
         \abs{
                \frac{
                d\mathcal{T}(x;w_{-i})
                }{
                    dx
                }\big\vert_{x=\mathcal{T}^{-1}(w_i; w_{-i})}
            }
        }
        \frac{
        e^{
            -\beta H(
                w, \theta
            )
        }
        }{
             \int dw^{\prime}_{-i}\
            e^{
            -\beta H(  
            w^{\prime}_{-i},
            \mathcal{T}(
            \mathcal{T}^{-1}(w_i; w_{-i})
            ; w^{\prime}_{-i}
            ),
            \theta)
            }
        }
        p(\mathcal{T}^{-1}(w_i; w_{-i}))
        \\
        &=
        \frac{1}{
         \abs{
                \frac{
                d\mathcal{T}(x;w_{-i})
                }{
                    dx
                }\big\vert_{x=\hat{x}(0;w)}
            }
        }
        \frac{
        e^{
            -\beta H(
                w, \theta
            )
        }
        }{
             \int dw^{\prime}_{-i}\
            e^{
            -\beta H(  
            w^{\prime}_{-i},
            \mathcal{T}(
            \hat{x}(0;w)
            ; w^{\prime}_{-i}
            ),
            \theta)
            }
        }
        p(\hat{x}(0;w)).
    \end{align*}
    
    For the reparameterized version, we have $w_{i}= \mathcal{T}(x_0;w_{-i})$.
    Then, we define the map
    $
    \mathcal{G}:(w_{-i}, x_0) \mapsto (w_{-i}, \mathcal{T}(x_0; w_{-i}))
    $.
    \begin{align*}
        \Tilde{p}(w\vert \theta)
        &=
        \abs{
        \nabla_{w_{-i}, w_i}\ \mathcal{G}^{-1}(w_{-i}, w_i) 
        }
        \Tilde{p}\left(
        \mathcal{G}^{-1}
        (w_{-i}, w_i)
        \vert \theta
        \right)
        \\
        &=
        \abs{
        \frac{
            d\mathcal{T}^{-1}(w_i; w_{-i})
        }{
            dw_i
        }
        }
        \Tilde{p}
        \left(w_{-i}, \mathcal{T}^{-1}(w_i; w_{-i})\vert\theta
        \right)
        \\
        &=
        \abs{
        \frac{
            d\mathcal{T}^{-1}(w_i; w_{-i})
        }{
            dw_i
        }
        }
        \Tilde{p}
        \left(w_{-i}\vert \mathcal{T}^{-1}(w_i; w_{-i}), \theta
        \right)
        p(\mathcal{T}^{-1}(w_i; w_{-i}))
        \\
        &=
         \frac{1}{
         \abs{
                \frac{
                d\mathcal{T}(x;w_{-i})
                }{
                    dx
                }\big\vert_{x=\mathcal{T}^{-1}(w_i; w_{-i})}
            }
        }
         \Tilde{p}
         \left(w_{-i}\vert \mathcal{T}^{-1}(w_i; w_{-i}), \theta
         \right)
        p(\mathcal{T}^{-1}(w_i; w_{-i}))
        \\
        &=
        \frac{1}{
         \abs{
                \frac{
                d\mathcal{T}(x;w_{-i})
                }{
                    dx
                }\big\vert_{x=\mathcal{T}^{-1}(w_i; w_{-i})}
            }
        }
         \frac{
        e^{
            -\beta \Tilde{H}(
                w_{-i}, \mathcal{T}^{-1}(w_i; w_{-i}), \theta
            )
        }
        }{
             \int dw^{\prime}_{-i}\
            e^{
            -\beta \Tilde{H}( 
            w^{\prime}_{-i},
            \mathcal{T}^{-1}(w_i; w_{-i}),
            \theta)
            }
        }
        p(\mathcal{T}^{-1}(w_i; w_{-i}))
        \\
        &=
        \frac{1}{
         \abs{
                \frac{
                d\mathcal{T}(x;w_{-i})
                }{
                    dx
                }\big\vert_{x=x_0}
            }
        }
         \frac{
        e^{
            -\beta H(w_{-i}, w_i, \theta)
        }
        }{
             \int dw^{\prime}_{-i}
            e^{
            -\beta H( 
            w^{\prime}_{-i},
           \mathcal{T}(x_0; w^{\prime}_{-i}),
            \theta)
            }
        }p(x_0).
    \end{align*}
    Due to the reparameterization, we have $\hat{x}(0;w)= x_0$. So we have $p(w\vert \theta)=\Tilde{p}(w\vert \theta)$.
\end{proof}

\section{Proposition 2}
\label{appendix:proof2}
\begin{prop}
\label{prop2}
    Maximizing the ELBO \qref{posterior_elbo} is equivalent to minimizing an upper bound of the KL divergence
    $
    D_{\text{KL}}(
    q_{\phi}(w)q_{\psi}(\theta)
    \Vert
    p(w, \theta\vert y)
    )
    $.
\end{prop}

\begin{proof}
We work with the following compact form of the ELBO:
\begin{align*}
    \operatorname{ELBO}(\phi, \psi, \chi\vert y)
    =
    \E{q_{\phi}(w)q_{\psi}(\theta) q_{\chi}(x_0)}{
        \log
        \left\{
        \dfrac{
        p(y\vert w, \theta)
        p(w\vert x_0, \theta) p(x_0, \theta)
        }{
        q_{\phi}(w)q_{\psi}(\theta) q_{\chi}(x_0)
        }
        \right\}
        }.
\end{align*}
It is easy to see this form is equivalent to \qref{posterior_elbo}.

We decompose the log evidence into three terms:
\begin{align*}
    \log p(y)
    &=
    \E{
        q_{\phi}(w) q_{\psi}(\theta) q_{\chi}(x_0)
    }{
    \log p(y)
    }
    \\
    &=
    \E{
        q_{\phi}(w) q_{\psi}(\theta) q_{\chi}(x_0)
    }{
        \log\left\{
        \dfrac{
            p\left(y, w, \theta, x_0\right)
        }{
            p\left(w, \theta, x_0\vert y\right)
        }
        \right\}
    }
    \\
    &=
    \E{
        q_{\phi}(w) q_{\psi}(\theta)q_{\chi}(x_0)
    }{
        \log\left\{
        \dfrac{
            p\left(y, w, \theta, x_0\right)
            q_{\phi}(w)q_{\psi}(\theta) q_{\chi}(x_0)
        }{
            p\left(w, \theta, x_0\vert y\right)
            q_{\phi}(w) q_{\psi}(\theta) q_{\chi}(x_0)
        }
        \right\}
    }
    \\
    &=
    \E{
        q_{\phi}(w)q_{\psi}(\theta)q_{\chi}(x_0)
    }{
        \log\left\{
        \dfrac{
            q_{\phi}(w)q_{\psi}(\theta) q_{\chi}(x_0)
        }{
            p\left(w, \theta, x_0\vert y\right)
        }
        \right\}
    }
    +
    \E{
        q_{\phi}(w)q_{\psi}(\theta)q_{\chi}(x_0)
    }{
        \log\left\{
        \dfrac{
            p\left(y, w, \theta, x_0\right)
        }{
            q_{\phi}(w) q_{\psi}(\theta) q_{\chi}(x_0)
        }
        \right\}
    }
    \\
    &=
    \E{
        q_{\phi}(w)q_{\psi}(\theta)
    }{
        \log\left\{
        \dfrac{
            q_{\phi}(w)q_{\psi}(\theta)
        }{
            p\left(w, \theta \vert y\right)
        }
        \right\}
    }
    +
    \E{
        q_{\phi}(w)q_{\psi}(\theta)q_{\chi}(x_0)
    }{
        \log\left\{
        \dfrac{
            q(x_0)
        }{
            p\left(x_0\vert w, \theta,y\right)
        }
        \right\}
    }
    +
    \operatorname{ELBO}(\phi, \psi, \chi\vert y)
    \\
    &=
    \operatorname{KL}\left[
    q_{\phi}(w)q_{\psi}(\theta)
    \Vert
    p(w, \theta \vert y)
    \right]
    +
    \E{q_{\phi}(w)q_{\psi}(\theta)}{
    \operatorname{KL}\left[
        q(x_0)\Vert
        p(x_0\vert w, \theta, y)
    \right]
    }
    +
    \operatorname{ELBO}(\phi, \psi, \chi\vert y).
\end{align*}
Since $\log p(y)$ is a constant, and 
$\E{q_{\phi}(w)q_{\psi}(\theta)}{
    \operatorname{KL}\left[
        q(x_0)\Vert
        p(x_0\vert w, \theta, y)
    \right]
    }$
is non-negative, we conclude the equivalence.
\end{proof}

\section{NSVI implementation details}
\label{appendix:nsvi}

We write the information Hamiltonian $H_1(w, \theta)$ using expectation, so it can be estimated using Monte Carlo methods. 
Let
\begin{align*}
    h_1(w, \theta, t)
    =
    \lVert
     \dot{\hat{x}}(t;w)-f(\hat{x}(t;w), t;\theta) \lVert^2,
\end{align*}
and $t$ follow the uniform distribution $p(t) = \mathcal{U}([0, T])$, we write
\begin{align*}
    H_1(w, \theta) = T\E{p(t)}
{h_1(w, \theta, t)}.
\end{align*}
So the log unnormalized relaxed physics-informed conditional prior is 
\begin{align}
\label{logpriortarget}
    \log \pi(w\vert x_0, \theta)
    =
    -\beta_1 T 
    \E{
    p(t)
    }
    {
    h_1(w, \theta, t)
    }
    -
    \beta_2 H_2(\hat{x}(0;w), x_0).
\end{align}

\subsection{SVI to approximate the relaxed physics-informed conditional prior $p(w\vert x_0, \theta)$}

We first describe the inner loop auxiliary stochastic variational inference.
We parameterize the guide $q_{\phi}(w)$ with the parameter $\phi$ to approximate the relaxed physics-informed conditional prior
$p(w\vert x_0, \theta)$ by maximizing the prior ELBO:
\begin{align*}
    \operatorname{ELBO}(\phi\vert x_0, \theta)
    =
    \E{
    q_{\phi}(w)
    }{
    \log
    \frac{
    \pi(w\vert x_0, \theta)
    }{
    q_{\phi}(w)
    }
    }.
\end{align*}
The log unnormalized relaxed physics-informed conditional prior is defined in \qref{logpriortarget}.

Applying Eq.~(\ref{logpriortarget}), the equivalent prior ELBO is
\begin{align*}
    \operatorname{ELBO}(\phi\vert x_0, \theta)
    =
    \E{
    q_{\phi}(w)
    p(t)
    }{
    -T\beta_1
    h(w, \theta, t)
    -\beta_2 H_2(\hat{x}(0;w), x_0)
    -\log q_{\phi}(w)
    }.
\end{align*}

Maximizing the prior ELBO requires computing its gradient with respect to the variational parameter $\phi$.
We apply the reparameterization trick~\citep{kingma2013auto}.
Specifically, we choose a base distribution $q(\epsilon)$ and a deterministic transformation $g_{\phi}$ parameterized by the same variational parameter $\phi$ such that
$g_{\phi}(\epsilon)\sim q_{\phi}$.
For more details, refer to section 2.4.4 in \citep{hao2024information}. 
Then the reparameterized ELBO is
\begin{align*}
    \operatorname{ELBO}(\phi\vert x_0, \theta)
    =
    \E{
    q(\epsilon)
    p(t)
    }{
    -T\beta_1
    h(g_{\phi}(\epsilon), \theta, t)
    -\beta_2 H_2(\hat{x}(0;g_{\phi}(\epsilon)), x_0)
    -\log q_{\phi}(g_{\phi}(\epsilon))
    }.
\end{align*}
Its gradient with respect to the variational parameter $\phi$ is
\begin{align*}
    \nabla_{\phi}
    \operatorname{ELBO}(\phi\vert x_0, \theta)
    =
    \E{
    q(\epsilon)p(t)
    }{
    \nabla_{\phi}
    \left[
    -T\beta_1
    h(g_{\phi}(\epsilon), \theta, t)
    -\beta_2 H_2(\hat{x}(0;g_{\phi}(\epsilon)), x_0)
    -\log q_{\phi}(g_{\phi}(\epsilon))
    \right]
    }.
\end{align*}
Let $\epsilon_i$ and $t_{ij}$ be the random samples from $q(\epsilon)$ and $p(t)$, an unbiased estimator of this gradient is
\begin{align}
\label{grad_elbo_prior}
    \widehat{
    \nabla_{\phi}
    \operatorname{ELBO}(\phi\vert x_0, \theta)
    }
    =
    \frac{1}{N_{\epsilon}N_t}
    \sum_{i=1}^{N_{\epsilon}}
    \sum_{j=1}^{N_t}
    \nabla_{\phi}
    \left[
    -T\beta_1
    h(g_{\phi}(\epsilon_i), \theta, t_{ij})
    -\beta_2 H_2(\hat{x}(0;g_{\phi}(\epsilon_i)), x_0)
    -\log q_{\phi}(g_{\phi}(\epsilon_i))
    \right].
\end{align}

We summarize the algorithm in Algorithm~\ref{alg:svi_prior}.

\RestyleAlgo{ruled} 
\SetKwInput{KwInit}{Initialization}
\SetKwInput{KwInput}{Input}
\SetKwInput{KwReturn}{Return}
\SetKwInput{Kwalg}{SVI\_PRIOR}
\SetKwComment{Comment}{$\triangleright$}{}
\begin{algorithm}[tb]
\caption{
SVI approximation to $p\left(w\vert x_0, \theta \right)$.
}\label{alg:svi_prior}
\Kwalg{
(
\begin{tabbing}
\hspace{1em} \= $x_0 $, \hspace{2em} \=
 \# \textit{The auxiliary initial state.}
 \\
\> $\theta$, \> \# \textit{Model parameter.} 
\\
\> $\phi$, \> \# \textit{Initial variational parameter.}
\\
\> niter, \> \# \textit{The number of optimization iterations.}
\\
\> $(n_{\Tilde{\epsilon}}, n_{\Tilde{t}})$, \> \# \textit{Sample sizes.}
\\
)
\end{tabbing}
}
 \For{$\text{it}=0; \text{it}<\text{niter}; \text{it}=\text{it}+1$}{  
 Sample $\epsilon_i$ independently from $q(\epsilon)$\;
 Sample $t_{ij}$ independently from $\mathcal{U}([0,T])$\;
 Compute 
 $
  \widehat{
    \nabla_{\phi}
    \operatorname{ELBO}(\phi\vert x_0, \theta)
    }
 $ 
 using \qref{grad_elbo_prior}
    \;
  Update $\phi$ using Adam with the gradient estimate 
  $
  \widehat{
    \nabla_{\phi}
    \operatorname{ELBO}(\phi\vert x_0, \theta)
    }
 $.
} 
 \KwReturn{
 $
 \phi
 $
 \hspace{2em}
 \#
 \textit{Final variational parameter.}
 }
\end{algorithm}

\subsection{NSVI to approximate the marginal posterior $p(w, \theta\vert y)$}

As before, we apply the reparameterization trick to the ELBO defined in \qref{posterior_elbo}.
Let $(q(\epsilon), q(\eta), q(\zeta))$ be the base distributions and $(g_{\phi}, g_{\psi}, g_{\eta})$ be the corresponding deterministic transformations such that
$g_{\phi}(\epsilon)\sim q_{\phi}$, 
$g_{\psi}(\eta)\sim q_{\psi}$,
and $g_{\chi}(\zeta)\sim q_{\chi}$.
Then, the reparameterized ELBO is
\begin{equation}
\label{appendix_elbo}
\begin{aligned}
    \operatorname{ELBO}(\phi, \psi, \chi\vert y)
    =
    &+
    \E{q(\epsilon)q(\eta)
    p(\mathcal{I}_{m_d})
    }{
    \frac{n_d}{m_d}
    \sum_{i\in\mathcal{I}_{m_d}}
    \log p(y_i\vert g_{\phi}(\epsilon), g_{\psi}(\eta))
    }
    \\
    &+
    \E{q(\epsilon)q(\eta)q(\zeta)}
    {
   \log \pi(g_{\phi}(\epsilon)\vert g_{\chi}(\zeta), g_{\psi}(\eta))
    }
    \\
    &+
    \E{q(\eta)q(\zeta)}{
     \log
        p(g_{\chi}(\zeta), g_{\psi}(\eta))
    }
    \\
    &-
    \E{q(\epsilon)q(\eta)q(\zeta)}{
    \log  q_{\phi}(g_{\phi}(\epsilon))
        q_{\psi}(g_{\psi}(\eta))
        q_{\chi}(g_{\chi}(\zeta))
    }
    \\
    &- \E{q(\eta)q(\zeta)}{
    \log
    Z(g_{\chi}(\zeta), g_{\psi}(\eta))
    }
    \end{aligned}.
\end{equation}

Calculating the gradient of ELBO requires to differentiate through the log partition function, which has the formula~\citep{hao2024information}:
\begin{equation}
\label{gradient_conditional_prior}
\begin{aligned}
    \nabla_{\psi, \chi}
   \log
    Z(g_{\chi}(\zeta), g_{\psi}(\eta))
    =
    \E{
    p(\tilde{w}\vert  g_{\chi}(\zeta), g_{\psi}(\eta))
    }{
    \nabla_{\psi, \chi}
    \log\left[
    \pi(
    \Tilde{w}\vert g_{\chi}(\zeta), g_{\psi}(\eta)
    )
    \right]
    }.
\end{aligned}
\end{equation}
To evaluate it, we need to sample from the relaxed physics-informed conditional prior $p(\tilde{w}\vert  g_{\chi}(\zeta), g_{\psi}(\eta))$.
We will use Algorithm~\ref{alg:svi_prior} to build a surrogate sampler.

Applying \qref{logpriortarget} to \qref{gradient_conditional_prior} and substitute the result into the ELBO \qref{appendix_elbo}, the gradient of ELBO is
\begin{equation*}
\begin{aligned}
    \nabla_{\phi, \psi, \chi}\operatorname{ELBO}(\phi, \psi, \chi\vert y)
    =
    &+
    \E{q(\epsilon)q(\eta)
    p(\mathcal{I}_{m_d})
    }{
    \frac{n_d}{m_d}
    \sum_{i\in\mathcal{I}_{m_d}}
    \nabla_{\phi, \psi}
    \log p(y_i\vert g_{\phi}(\epsilon), g_{\psi}(\eta))
    }
    \\
    &-
    \E{q(\epsilon)q(\eta)q(\zeta)p(t)}
    {
        \nabla_{\phi, \psi, \chi}
        \left[
        T\beta_1
        h(g_{\phi}(\epsilon), g_{\psi}(\eta), t)
        +
        \beta_2 H_2(\hat{x}(0;g_{\phi}(\epsilon)), g_{\chi}(\zeta))
        \right]
    }
    \\
    &+
    \E{q(\eta)q(\zeta)}{
    \nabla_{\psi, \chi}
     \log
        p(g_{\chi}(\zeta), g_{\psi}(\eta))
    }
    \\
    &-
    \E{q(\epsilon)q(\eta)q(\zeta)}{
    \nabla_{\phi, \psi, \chi}
    \log  q_{\phi}(g_{\phi}(\epsilon))
        q_{\psi}(g_{\psi}(\eta))
        q_{\chi}(g_{\chi}(\zeta))
    }
    \\
    &+ \E{
    q(\eta)q(\zeta)
    p(\tilde{w}\vert g_{\chi}(\zeta), g_{\psi}(\eta))
    p(\tilde{t})
    }{
   \nabla_{\psi, \chi}
        \left[
        T\beta_1
        h(\Tilde{w}, g_{\psi}(\eta), \Tilde{t})
        +
        \beta_2 H_2(\hat{x}(0;\Tilde{w}), g_{\chi}(\zeta))
        \right]
    }
    \end{aligned}.
\end{equation*}

To build an unbiased estimator for the ELBO's gradient, we sample $(\epsilon_i, \eta_i, \zeta_i)$ from $(q(\epsilon), q(\eta), q(\zeta))$, $t_{ij}$ from $p(t)$, $\tilde{w}_{ik}$ from $p(\Tilde{w}\vert g_{\chi}(\zeta_{i}), g_{\psi}(\eta_i))$, $\Tilde{t}_{ikl}$ from $p(\tilde{t})$, and only subsample one dataset with the index set $I_{m_d}$ to get
\begin{equation}
\label{grad_estimator_elbo_post}
\begin{aligned}
    \widehat{
    \nabla_{\phi, \psi, \chi}
     \operatorname{ELBO}(\phi, \psi, \chi\vert y)
    }
    =
    &+
    \frac{n_d}{
    N_{\epsilon\eta\zeta}m_d
    }
    \sum_{i=1}^{N_{\epsilon\eta\zeta}}
        \sum_{n\in I_{m_d}}
        \nabla_{\phi, \psi}
        \log 
        p(
        y_n\vert 
        g_{\phi}(\epsilon_i),
        g_{\psi}(\eta_i)
        )
    \\
    &-
    \frac{1}{
    N_{\epsilon\eta\zeta}
    N_t
    }
    \sum_{i=1}^{N_{\epsilon\eta\zeta}}
    \sum_{j=1}^{N_t}
        \nabla_{\phi, \psi, \chi}
        \left[
        T\beta_1
        h(g_{\phi}(\epsilon_i), g_{\psi}(\eta_i), t_{ij})
        +
        \beta_2 H_2(\hat{x}(0;g_{\phi}(\epsilon_i)), g_{\chi}(\zeta_i))
        \right]
    \\
    &+
    \frac{1}{
    N_{\epsilon\eta\zeta}
    }
    \sum_{i=1}^{N_{\epsilon\eta\zeta}}
    \nabla_{\psi, \chi}
        \log
        p(g_{\chi}(\zeta_i), g_{\psi}(\eta_i))
    \\
    &-
    \frac{1}{
    N_{\epsilon\eta\zeta}
    }
    \sum_{i=1}^{N_{\epsilon\eta\zeta}}
    \nabla_{\phi, \psi, \chi}
        \log\left[
            q_{\phi}(g_{\phi}(\epsilon_i))
            q_{\psi}(g_{\psi}(\eta_i))
            q_{\chi}(g_{\chi}(\zeta_i))
        \right]
    \\
    &+
    \frac{1}{
    N_{\epsilon\eta\zeta}
    N_{\Tilde{w}}
    N_{\Tilde{t}}
    }
    \sum_{i=1}^{N_{\epsilon\eta\zeta}}
    \sum_{k=1}^{N_{\Tilde{w}}}
    \sum_{l=1}^{N_{\Tilde{t}}}
    \nabla_{\psi, \chi}
        \left[
        T\beta_1
        h(\Tilde{w}_{ik}, g_{\psi}(\eta_i), \Tilde{t}_{ikl})
        +
        \beta_2 H_2(\hat{x}(0;\Tilde{w}_{ik}), g_{\chi}(\zeta_i))
        \right]
\end{aligned}   
\end{equation}

To compute this estimator, we have to draw $\Tilde{w}_{ik}$ from the relaxed physics-informed conditional prior $p(\Tilde{w}\vert g_{\chi}(\zeta_{i}), g_{\psi}(\eta_i))$.
We run Algorithm~\ref{alg:svi_prior} to get surrogate samples with the optimized guide 
$q_{\Tilde{\phi}_i}(\Tilde{w})$.
We call this inner loop auxiliary stochastic variational inference.
For the computational efficiency purpose, we draw only one sample from $q(\epsilon)$, $q(\eta)$, and $q(\zeta)$, i.e., $N_{\epsilon\eta\zeta}=1$.
Additionally, we run Algorithm~\ref{alg:svi_prior} in a non-convergent and persistent manner.
This means that instead of running the inner loop auxiliary algorithm to full convergence, we perform only a few updates, e.g., ten steps, and initialize the inner loop auxiliary algorithm at the next optimization iteration using the trained auxiliary guide parameter at the current iteration step.

Algorithm~\ref{alg:svi_post} summarizes the steps to approximate the marginal posterior distribution.

\RestyleAlgo{ruled} 
\SetKwInput{KwInit}{Initialization}
\SetKwInput{KwInput}{Input}
\SetKwInput{KwReturn}{Return}
\SetKwInput{Kwalg}{NSVI\_POSTERIOR}
\SetKwComment{Comment}{$\triangleright$}{}
\begin{algorithm}[tb]
\caption{
NSVI approximation to $p\left(w, \theta \vert y\right)$.
}\label{alg:svi_post}
\Kwalg{
(
\begin{tabbing}
\hspace{1em} \= $\left(\phi, \psi, \chi, \Tilde{\phi}\right)$, \hspace{4em} \=
 \# \textit{Variational parameters.}
\\
\> (niter, niter\_auxi), \> \# \textit{The number of optimization iterations.}
\\
\> $(n_{\epsilon\eta\zeta},n_t,  \tilde{n}_\epsilon, \tilde{n}_t, m_y)$, \> \# \textit{Sample sizes.}
\\
)
\end{tabbing}
}
 \For{$\text{it}=0; \text{it}<\text{niter}; \text{it}=\text{it}+1$}{
 Sample $(\epsilon_1. \eta_1, \zeta_1) \sim q$,
 \hspace{2em}
 \# Assuming $n_{\epsilon\eta\zeta} = 1$ for computational efficiency.
 \\
 Compute $
 (\theta_1, x_{0, 1}) = 
 (
 g_{\psi}\left(\eta_1\right)
 ,
 g_{\chi}(\zeta_1)
 )
 $
\;
$\Tilde{\phi}\gets$ \textbf{SVI\_PRIOR}
$
\left(
x_{0, 1},
\theta_{1},
\tilde{\phi},
\text{niter\_auxi},
\left(\tilde{n}_\epsilon, \tilde{n}_t\right)
\right),
$
 \hspace{2em}
\# Run \textbf{Algorithm} \ref{alg:svi_prior}.
\\
Sample $\tilde{w}_{1k}$ independently from $q_{\tilde{\phi}}$\;
Sample $t_{1j}$ independently from $\mathcal{U}([0,T])$\;
Sample $\Tilde{t}_{1kl}$ independently from $\mathcal{U}([0,T])$\;
Compute 
$ \widehat{
    \nabla_{\phi, \psi, \chi}
    \operatorname{ELBO}\left(
        \phi, \psi, \chi \vert y
        \right)
    }
$ using \qref{grad_estimator_elbo_post}\;
Update $(\phi, \psi, \chi)$ using Adam with the gradient estimate 
$
\widehat{
    \nabla_{\phi, \psi, \chi}
    \operatorname{ELBO}\left(
        \phi, \psi, \chi \vert y
        \right)
    }
$.
 }
 \KwReturn{
$\left(q_{\phi}(w), q_{\psi}(\theta)\right)$
 \hspace{2em}
 \#
 \textit{Approximate marginal posterior distribution.}
 }
\end{algorithm}

\section{NPSGLD implementation details}
\label{appendix: npsgld}

In our numerical experiments, the adaptively-updated rules in Eqs. (\ref{RMSprop_prior}) and (\ref{RMSprop2}) generally accelerate convergence during the initial phase of MCMC sampling, where larger update step sizes are beneficial.
However, as the MCMC chains approach the target probability density modes, smaller, stable update step sizes become necessary, and the adaptive steps can introduce instability.
To address this, we gradually anneal the memory size parameter $\alpha$ to 1.
Numerically, we create a monotonically increasing vector $\boldsymbol{\alpha}$, which anneals the memory size parameter over a user-specified number of iterations.
This approach allows us to adaptively precondition the SGLD in the initial sampling phase, enabling faster exploration of the parameter space. 
When $\alpha=1$, the SGLD is statically preconditioned, stabilizing the MCMC chains as they settle into the target density modes.

\subsection{PSGLD to sample from the relaxed physics-informed conditional prior $p(w\vert x_0, \theta)$}
We use PSGLD to sample from the relaxed physics-informed conditional prior
$p(w\vert x_{0}, \theta)$.
The update step is
\begin{align}
\label{psgld_update_prior}
        \Delta w_k
    =
   \rho_k
    \left[
        M(w_k)
            \nabla_{w}
            \log \pi(w_k\vert x_{0}, \theta)
        +\Gamma(w_k)
    \right]
    +
    M(w_k)^{\frac{1}{2}}
    \sqrt{2\rho_k} \xi_{k},
\end{align}
where, recall from Eq.~(\ref{logpriortarget})
\begin{align*}
    \nabla_{w}
            \log \pi(w_k\vert x_0, \theta)
    =
    -T\beta_1
    \E{
    t\sim \mathcal{U}([0, T])
    }{
    \nabla_{w}
    h(w_k, \theta, t)}
    -\beta_2 
    \nabla_{w}
    H_2(\hat{x}(0; w_k), x_0)
    .
\end{align*}
We use an unbiased estimator for it:
\begin{align}
    \widehat{\nabla_{w}
            \log \pi(w_k\vert x_0, \theta)
            }
    =
    \frac{-T\beta_1}{n_t}
    \sum_{i=1}^{n_t}
    \nabla_w h(w_k, \theta, t_i)
    -\beta_2 
    \nabla_w
    H_2(\hat{x}(0; w_k), x_0).
    \label{psgld_grad_log_cond_prior}
\end{align}
The precondition matrix $M(w_k)$ update rule is 
\begin{equation}
\label{RMSprop_prior}
    \begin{split}     
    V(w_k)
    &=
    \alpha V(w_{k-1})
    +
    (1-\alpha)
    g(w_k)\odot g(w_k),
    \\
    M(w_k)
    &=
    \operatorname{diag}
    \left(
        \frac{1}{
        \delta +
        \sqrt{
        V(w_k)
        }
        }
    \right),
    \end{split}
\end{equation}
where
\begin{align*}
    g(w_k)
    =
    \widehat{
    \nabla_{w}
            \log \pi(w_k\vert x_0, \theta)
        }.
\end{align*}

We summarize these steps in Algorithm~\ref{alg:PSGLD_prior}.

\RestyleAlgo{ruled} 
\SetKwInput{KwInit}{Initialization}
\SetKwInput{KwInput}{Input}
\SetKwInput{KwReturn}{Return}
\SetKwInput{Kwalg}{PSGLD\_PRIOR}
\SetKwComment{Comment}{$\triangleright$}{}
\begin{algorithm}[tb]
\caption{
PSGLD to sample from $p\left(w\vert x_0, \theta \right)$.
}\label{alg:PSGLD_prior}
\Kwalg{
(
\begin{tabbing}
\hspace{1em} \= $w_0 $, \hspace{2em} \=
 \# \textit{MCMC chain initial state.}
 \\
\> $(x_0, \theta)$, \> \# \textit{the auxiliary initial state and model parameter.} 
\\
\> niter, \> \# \textit{The number of MCMC iterations.}
\\
\> $n_t$, \> \# \textit{Sample size.} 
\\
\> $\rho$, \> \# \textit{Step size.}
\\
\> $(\boldsymbol{\alpha}, \delta)$, \> \# \textit{RMSprop parameters.}
\\
\> $(V, M)$, \> \# \textit{Initial RMSprop matrices.}
\\
)
\end{tabbing}
}
 \For{$k=0; k<\text{niter}; k=k+1$}{  
 Sample $t_i$ independently from $\mathcal{U}([0,T])$\;
 Compute 
 $
\widehat{
        \nabla_{w}
            \log \pi(w_k\vert x_0, \theta)
    }
$ using \qref{psgld_grad_log_cond_prior}
\;
Choose the current memory size $\alpha$ from $\boldsymbol{\alpha}$.
 Update 
$V(w_k)$ and $M(w_k)$
using \qref{RMSprop_prior}
\;
Compute 
$\Delta w_k$ using \qref{psgld_update_prior} and update $w_{k+1}$.
} 
 \KwReturn{
 $
 \left(w_{-1}, V, M\right)
 $
 \hspace{2em}
 \#
 \textit{The final state of MCMC chain and RMSprop parameters.}
 }
\end{algorithm}

\subsection{NPSGLD to sample from the marginal posterior $p(w, \theta\vert y)$}

The update rule in \qref{npsgld_update} requires calculating the gradient of $\log Z(x_{0, k}, \theta_k)$.
Similar to the NSVI case, the gradient is
\begin{align*}
    \nabla_{\theta, x_0} 
    \log Z( x_{0, k}, \theta_k)
    &=
     \E{
    p(\tilde{w}\vert  x_{0, k}, \theta_k)
    }{
    \nabla_{\theta, x_0}
    \log\left[
    \pi(
    \Tilde{w}\vert x_{0, k}, \theta_k
    )
    \right]
    }.
\end{align*}
So the gradient of log relaxed physics-informed conditional prior is
\begin{align}
\label{appendix:grad_prior_npsgld}
    \nabla_{w, \theta, x_0}
    \log p(w_k\vert  x_{0, k}, \theta_k)
    =
    \nabla_{w, \theta, x_0} 
    \log \pi
    (
    w_k\vert x_{0, k}, \theta_k
    )
    -
    \E{
    p(\tilde{w}\vert  x_{0, k}, \theta_k)
    }{
    \nabla_{\theta, x_0}
    \log\left[
    \pi(
    \Tilde{w}\vert x_{0, k}, \theta_k
    )
    \right]
    }
\end{align}

We apply \qref{logpriortarget} to \qref{appendix:grad_prior_npsgld}, and sample $t_l$ from $p(t)$, $\tilde{w}_m$ from $p(\tilde{w}\vert  x_{0, k}, \theta_k)$, and
$\tilde{t}_{mn}$ from $p(\Tilde{t})$ to get the estimator 
\begin{equation}
\label{psgld_condi_grad}
    \begin{aligned}
    \widehat{\nabla_{w, \theta, x_0} 
    \log p(w_k\vert x_{0, k}, \theta_k)}
    =
    &-
    \frac{T\beta_1}{n_t}
    \sum_{l=1}^{n_t}
        \nabla_{w, \theta}
        h(w_k, \theta_k, t_l)
    -\beta_2
     \nabla_{w, x_0}
     H_2(
    \hat{x}(0; w_k), x_{0,k}
    )
    \\
    &+
    \frac{T\beta_1}{n_{\Tilde{w}}n_{\Tilde{t}}}
    \sum_{m=1}^{n_{\Tilde{w}}}
               \sum_{n=1}^{n_{\Tilde{t}}}
               \left[
               \nabla_{\theta}
               h(\Tilde{w}_{m}, \theta_k, \Tilde{t}_{mn})      
            +
            \beta_2
            \nabla_{x_0}
            H_2(
            \hat{x}(0; \Tilde{w}_m), x_{0,k}
            )
            \right]
    \end{aligned}.
\end{equation}

The precondition matrix is
\begin{equation}
\label{RMSprop2}
    \begin{split}     
    V(w_k, \theta_k, x_{0,k})
    &=
    \alpha V(w_{k-1}, \theta_{k-1}, x_{0,k-1})
    +
    (1-\alpha)
    g(w_{k}, \theta_{k}, x_{0,k})\odot g(w_{k}, \theta_{k}, x_{0,k}),
    \\
    M(w_k, \theta_k, x_{0,k})
    &=
    \operatorname{diag}
    \left(
        \frac{1}{
        \delta +
        \sqrt{
        V(w_k, \theta_k, x_{0,k})
        }
        }
    \right).
    \end{split}
\end{equation}
In the above equation, we define
\begin{align}
\label{npsgld_g}
    g(w_{k}, \theta_{k}, x_{0,k})
    =
    \frac{n_d}{m_d}
    \sum_{i=1}^{m_d}
    \nabla_{w, \theta}
    \log
    p(y_{ki}\vert w_k, \theta_k)
    +
    \widehat{
    \nabla_{w, \theta, x_0}
    \log
    p(w_k\vert x_{0,k}, \theta_k)
    }
    +
    \nabla_{\theta, x_0}
    \log
    p(x_{0, k}, \theta_k).
\end{align}

Similar to NSVI, we have to sample from the relaxed physics-informed conditional prior
$p(\Tilde{w}\vert x_{0,k}, \theta_k)$ to obtain  samples $\Tilde{w}_{m}$.
We use Algorithm~\ref{alg:PSGLD_prior} to sample from this prior by running an auxiliary MCMC chain.
This means that we only consider the case where $n_{\Tilde{w}}=1$.

To improve the computational efficiency, we run NPSGLD in a persistent, short-run, non-convergent manner for the auxiliary chain.
Namely, we initialize the auxiliary chain at the next iteration using the result from the current iteration and update it for only a few steps, e.g., ten steps.
Algorithm~\ref{alg:PSGLD_post} outlines the process.

\RestyleAlgo{ruled} 
\SetKwInput{KwInit}{Initialization}
\SetKwInput{KwInput}{Input}
\SetKwInput{KwReturn}{Return}
\SetKwInput{Kwalg}{NPSGLD\_POSTERIOR}
\SetKwComment{Comment}{$\triangleright$}{}
\begin{algorithm}[tb]
\caption{
NPSGLD to sample from $p\left(w, \theta\vert y \right)$.
}\label{alg:PSGLD_post}
\Kwalg{
(
\begin{tabbing}
\hspace{1em} \= $(w_0, x_{0,0}, \theta_0)$, \hspace{2em} \=
 \# \textit{MCMC chain initial states.}
\\
\> $\Tilde{w}_0$, \> \# \textit{Auxiliary MCMC chain initial state.}
\\
\> (niter, niter\_auxi), \> \# \textit{The number of (auxiliary) MCMC iterations.}
\\
\> $(n_{t}, n_{\Tilde{t}}, n_{\Tilde{w}}, m_y)$ \> \# \textit{Sample sizes.} 
\\
\> $(\rho, \Tilde{\rho})$, \> \# \textit{Step sizes.}
\\
\> $(\boldsymbol{\alpha}, \delta, \boldsymbol{\Tilde{\alpha}}, \Tilde{\delta})$, \> \# \textit{RMSprop parameters.}
\\
\> $(V, M, \Tilde{V}, \Tilde{M})$, \> \# \textit{Initial RMSprop matrices.}
\\
)
\end{tabbing}
}
 \For{$k=0; k<\text{niter}; k=k+1$}{  
 \# Assuming $n_{\Tilde{w}}=1$ for computational efficiency.
 \\
 Choose the current memory size $\Tilde{\alpha}$ from $\boldsymbol{\Tilde{\alpha}}$;
 \\
 $(\Tilde{w}_0, \Tilde{V}, \Tilde{M}) \gets $ \textbf{PSGLD\_PRIOR}
 $
 \left(
    \Tilde{w}_0,
    (x_{0, k}, \theta_k),
    \text{niter\_auxi},
    n_{\Tilde{t}},
    \Tilde{\rho},
    (\Tilde{\alpha}, \Tilde{\delta}),
    (\Tilde{V}, \Tilde{M})
 \right),
 $
 \hspace{1em}
\# Run \textbf{Algorithm} \ref{alg:PSGLD_prior}.
 \\
 Sample $t_{l}$ and $t_{1m}$ independently from $\mathcal{U}([0,T])$\;
 Compute
 $
\widehat{\nabla_{w, \theta, x_0} 
    \log p(w_k\vert \theta_k, x_{0, k})}
$
using Eq.~(\ref{psgld_condi_grad}) with $\Tilde{w}_{m}=\Tilde{w}_0$
 \;
 Subsample a minibatch dataset $y_{i\in I_{m_d}}$ from $y$\;
 Choose the current memory size $\alpha$ from $\boldsymbol{\alpha}$.
Update $V(w_k, \theta_k, x_{0,k})$ and $M(w_k, \theta_k, x_{0,k})$ using Eqs.~(\ref{RMSprop2})
 \;
 Compute $(
        \Delta w_k, 
        \Delta \theta_k,
        \Delta x_{0,k}
    )$
    using \qref{npsgld_update}
    and update $(w_{k+1}, \theta_{k+1}, x_{0,k+1})$.
}
 \KwReturn{
 $
 \left(w_{1:k}, x_{0, 1:k}, \theta_{1:k}\right)
 $
 \hspace{2em}
 \#
 \textit{MCMC samples.}
 }
\end{algorithm}

\section{State filtering}

Algorithms \ref{alg:svi_post} and \ref{alg:PSGLD_post} are developed to handle the complex task of joint state and parameter estimation. However, if only state filtering is required and the model parameters are known, these algorithms remain applicable.

\section{Computational time}
\label{appendix: computational time}
Tables \ref{tab:1}, \ref{tab:2}, \ref{tab:3} and \ref{tab:4} summarize the dimension of $w$ in $\hat{x}(t;w)$ and computational time for all experiments.
Our hardware is an Apple M1 Pro chip with 10 CPU cores.
Overall, we find NSVI is faster than NPSGLD, albeit less accurate than NPSGLD.

\begin{table}[!htb]
\caption{computational time of examples in section~\ref{example 1}.}
\centering
\label{tab:1}
\begin{tabular}{ccccc}
\toprule[0.5mm]
& Repara NSVI & Relaxed NSVI & Relaxed NSGLD & Relaxed NPSGLD
\\
\midrule[0.25mm]
$w$ dimension
& 162& 162& 162&162
\\
Computational time (s)
&8&7& 77&90
\\
\bottomrule[0.5mm]
\end{tabular}
\end{table}

\begin{table}[!htb]
\caption{computational time of examples in section~\ref{section:2dof}.}
\centering
\label{tab:2}
\begin{tabular}{ccccc}
\toprule[0.5mm]
& 
\begin{tabular}[c]{@{}c@{}}w/o residual path: \\ NSVI\end{tabular}
&
\begin{tabular}[c]{@{}c@{}}w/o residual path: \\ NPSGLD\end{tabular}
& \begin{tabular}[c]{@{}c@{}}w/ residual path: \\ NSVI\end{tabular}
& \begin{tabular}[c]{@{}c@{}}w/ residual path: \\ NPSGLD\end{tabular}
\\
\midrule[0.25mm]
$w$ dimension
&80 &80 & 344&344
\\
Computational time (s)
&5&25&26&195
\\
\bottomrule[0.5mm]
\end{tabular}
\end{table}
\begin{table}[!htb]
\caption{computational time of examples in section~\ref{section:high}.}
\centering
\label{tab:3}
\begin{tabular}{ccc}
\toprule[0.5mm]
& NSVI
& NPSGLD
\\
\midrule[0.25mm]
$w$ dimension
&4660 & 4660
\\
Computational time (s)
&196&325
\\
\bottomrule[0.5mm]
\end{tabular}
\end{table}
\begin{table}[!htb]
\caption{computational time of examples in section~\ref{section:nes}.}
\centering
\label{tab:4}
\begin{tabular}{cc}
\toprule[0.5mm]
& NSVI
\\
\midrule[0.25mm]
$w$ dimension
&16000 
\\
Computational time (s)
&827
\\
\bottomrule[0.5mm]
\end{tabular}
\end{table}

\bibliographystyle{unsrtnat}

\end{document}